\documentclass{article}

\usepackage{arxiv}
\makeatletter
\AtBeginDocument{%
  \pagestyle{fancy}%
  \rhead{}%
}

\renewcommand{\keywords}[1]{\par\noindent\textbf{Keywords:} #1}
\makeatother
\usepackage{booktabs}
\usepackage{tabularx}
\usepackage[utf8]{inputenc} 
\usepackage[T1]{fontenc}    
\usepackage{hyperref}       
\usepackage{url}            
\usepackage{orcidlink}     
\usepackage{nicefrac}       
\usepackage{microtype}      
\usepackage{graphicx}
\usepackage{subcaption}
\usepackage{natbib}
\usepackage{lineno}         

\usepackage{amsmath,amssymb,amsfonts,amsthm,mathtools}
\newcommand{\R}{\mathbb{R}}
\usepackage{array}
\usepackage{algorithm}
\usepackage{algorithmic}
\usepackage{enumitem}
\usepackage{pifont}
\usepackage{tikz}
\usepackage{tikz-cd}
\usetikzlibrary{calc,positioning,decorations.pathmorphing,arrows.meta,angles,quotes}

\newtheorem{theorem}{Theorem}
\newtheorem{definition}{Definition}
\newtheorem{proposition}{Proposition}
\newtheorem{lemma}{Lemma}
\newtheorem{corollary}{Corollary}
\newtheorem{remark}{Remark}
\newcommand{\Yhat}{\widehat{\mathcal Y}}

\newcommand{\Roles}{\mathcal{R}}
\newcommand{\Agents}{\mathcal{A}}

\newtheorem{assumption}{Assumption}
\newcommand{\X}{\mathcal{X}}
\newcommand{\Y}{\mathcal{Y}}
\newcommand{\Tcal}{\mathsf Y}

\newcommand{\smallclubsuit}{\mathord{\scriptscriptstyle\clubsuit}}

\usepackage{authblk}

\makeatletter
\renewcommand\AB@affilnote[1]{\textsuperscript{#1}\,}
\makeatother

\title{Tree-Based Formalization of Multi-Agent Complementarity in Human--AI Interactions}

\author[1,2,3]{Andrea Ferrario\orcidlink{0000-0001-9968-9474}\thanks{Email: \texttt{aferrario@ethz.ch}.} }
\affil[1]{Institute of Biomedical Ethics and History of Medicine, University of Zurich, Zurich, Switzerland}
\affil[2]{SUPSI, Dalle Molle Institute for Artificial Intelligence (IDSIA), Lugano, Switzerland}
\affil[3]{ETH Zurich, Zurich, Switzerland}

\begin{document}
\maketitle

\begin{abstract}
Predictive tasks constitute subprocedures within human--AI interactions (HAIs) involving professionals and AI systems. Complementarity occurs when an HAI outperforms a benchmark defined by its constituent agents. Existing formalizations, however, focus on a single human and a single AI, leaving multi-agent interactions insufficiently characterized. We introduce a tree-based mathematical framework for multi-agent complementarity. Ordered agent-role configurations decorate the leaves of rooted planar binary trees and local aggregation rules generate protocol outputs evaluated against either a best-average or a pointwise-min benchmark.

We establish four results. First, under unrestricted linear pooling, every protocol tree provides a bijective coordinate system for the relative interior of the simplex of global leaf weights. Arbitrary trees with the same ordered leaves admit output- and complementarity-preserving reparameterizations, realized locally through path-independent Tamari rotations. However, HAI constraints motivated by professional hierarchies or epistemic confidence can violate admissibility preservation between equivalent protocols. Second, for regression under squared loss, we characterize complementarity through residual correction and derive optimal linear-pooling weights, showing how predictive disagreement can improve performance. We characterize protocol-indifference loci under identical local parameters. Third, selector-based protocols cannot achieve complementarity relative to the pointwise-min benchmark, independently of task and loss. Fourth, in binary classification, internal aggregation cannot achieve complementarity under endpoint-monotone losses and the pointwise-min benchmark, with an analogous obstruction for multiclass cross-entropy. Numerical illustrations examine the main theoretical claims. Our framework provides a mathematical foundation for analyzing multi-agent predictive subprocedures within HAIs, establishing a first step toward understanding how interaction, aggregation, and real-world constraints affect HAI predictive performance.
\end{abstract}


\section{Introduction}
\label{section:intro}
Complementarity is a key concept in human--AI interactions. It is the case in which the human--AI team performs better than either component alone. While complementarity has attracted much attention across HAI over the years \citep{Bansal2021CHI,Bansal2021AAAI,hemmer2021human,Hemmer2025EJIS,Donahue2022FAccT,Vaccaro2024NHB,campagner2026humansmachinestrulycomplement,ferrario2026complementaritygroundtruth}, two questions remain insufficiently addressed. First, complementarity is evaluated against the better aggregate empirical loss of the human or the AI over a dataset; this benchmark is problematic in high-risk  settings, e.g, medicine or cybersecurity, and motivates considering alternative,
more demanding benchmarks that account for heterogeneity in agent performance across cases. Second, formal work on complementarity remains limited, with approaches centered on the two-agent setting: a human and an AI system interacting on the same prediction task \citep{Hemmer2025EJIS,Donahue2022FAccT}. 
However, real-world predictive workflows may involve several human professionals and AI systems contributing predictions at different stages. In such settings, complementarity requires specifying how agent predictions are ordered and composed, how intermediate predictions contribute to a final output, and which context-dependent benchmark is used to evaluate the interaction.

We address these questions by introducing a tree-based formalization of complementarity for multi-agent prediction-task HAIs. We consider interactions in which $N\ge 2$ human and AI agents provide fixed prediction vectors for the same prediction task and dataset. A multi-agent HAI protocol is represented by an ordered agent-role configuration together with a rooted planar binary tree whose $N$ leaves are decorated by the agents' prediction vectors. The ordered configuration specifies the agent-role order; the tree represents the recursive composition
structure; and a local binary rule determines how two predictions
are combined at each internal node. Evaluating these rules
recursively yields the final protocol prediction, whose empirical
loss determines complementarity relative to a chosen benchmark.
Our contributions are as follows.

\begin{enumerate}[leftmargin=*,itemsep=3pt]

\item \textbf{Benchmark-relative multi-agent complementarity.}
We introduce a tree-based formalization of complementarity for prediction-task HAIs with arbitrarily many agents, separating the protocol-dependent interaction term from the benchmark term. We consider two context-relative benchmarks: the best-average benchmark, which compares the protocol against the best fixed agent over the dataset \citep{Donahue2022FAccT,ferrario2026epistemology}, and the pointwise-min benchmark, which compares it against the best available agent prediction on each case before averaging.

\item \textbf{Universal protocol equivalence under linear pooling.}
We consider local linear pooling, in which two predictions $u,v$
are combined through a convex combination $m_{2,\alpha}(u,v)=\alpha u+(1-\alpha)v$, with $\alpha\in(0,1)$. For every $N$, we show that each protocol tree provides a bijective coordinate system for the interior of the simplex of global leaf weights, independently of the prediction task
and loss function. Consequently, any two trees with the same ordered leaves admit an output-preserving reparameterization. This establishes the equivalence of their attainable predictions and complementarity values under unconstrained linear pooling.

\item \textbf{Tamari reparameterizations and universal protocol equivalence.}
We derive explicit rational transformations of local weights for elementary Tamari rotations \citep{tamari1962problemes} and establish their path independence for every $N$: different directed Tamari paths connecting the same endpoint trees induce the same reparameterization. These results provide a  local realization of universal protocol equivalence.  We then characterize admissibility preservation through feasible global-weight sets, examining
how interaction-level constraints can restrict universal protocol equivalence.

\item \textbf{Task-specific complementarity results.}
For regression under squared loss, we characterize complementarity through Euclidean geometry and derive optimization results for linear pooling. We show that selector-based protocols cannot achieve complementarity relative to the pointwise-min benchmark, independently
of the prediction task and loss function. For binary classification, we establish an impossibility result for internal aggregation under endpoint-monotone losses, including the relevant Bregman and Bernoulli
$f$-divergence losses. An analogous obstruction holds for multiclass cross-entropy, while non-internal aggregation rules can escape these restrictions \citep{neyman2023no}.

\item \textbf{Numerical illustrations.}
We provide numerical illustrations of the theoretical results, examining the geometry of complementarity in regression, protocol-indifference loci for distinct tree structures, and the possibility of achieving complementarity in binary classification through non-internal logit amplification.
\end{enumerate}

Our framework extends complementarity beyond two-agent interactions, providing a mathematical foundation for studying predictive tasks embedded as subprocedures in complex real-world HAI workflows. In such settings, constraints arising from professional hierarchies or epistemically justified confidence may limit the admissibility of otherwise equivalent protocols under linear pooling. Finally, our classification impossibility results motivate investigating alternative aggregation mechanisms and performance criteria, particularly in high-risk HAIs. The logical structure of the manuscript is summarized in Figure~\ref{fig:paper_roadmap}. 

\begin{figure}[t]
\centering
\begin{tikzpicture}[
    >=Latex,
    font=\footnotesize,
    every node/.style={outer sep=0pt},
    base/.style={
        draw=black!65,
        rounded corners=2pt,
        align=center,
        inner sep=5pt,
        line width=0.5pt
    },
    wide/.style={
        base,
        text width=.80\linewidth
    },
    branch/.style={
        base,
        text width=.37\linewidth,
        minimum height=17mm
    },
    keybox/.style={
        branch
    },
    arrow/.style={
        ->,
        draw=black!70,
        line width=0.6pt
    }
]

\node[wide] (related) {
    \textbf{Section~\ref{sec:related}: Related Work on Formalizations of Complementarity}\\[2pt]
};

\node[wide, below=5mm of related] (framework) {
    \textbf{Section~\ref{sec:tree}: Tree-Based Formalization of Multi-Agent
    Complementarity}\\[2pt]
    \begin{tabular}{@{}p{1.2em}@{}p{\dimexpr\linewidth-1.2em\relax}@{}}
        $\bullet$ & Protocol planar binary trees\\
        $\bullet$ & Benchmarks and multi-agent complementarity
    \end{tabular}
};

\node[branch, anchor=north west] (equiv)
    at ([yshift=-9mm]framework.south west) {
    \textbf{Section~\ref{section:linear_protocol_coordinates}}\\
    \textbf{Universal equivalence of tree-based protocols}\\[2pt]
    \begin{tabular}{@{}p{1.2em}@{}p{\dimexpr\linewidth-1.2em\relax}@{}}
        $\bullet$ & Protocol-tree coordinate systems\\
        $\bullet$ & Tamari coordinate changes and path independence
    \end{tabular}
};

\node[keybox, anchor=north west] (constraints)
    at ([yshift=-5mm]equiv.south west) {
    \textbf{Subsection~\ref{subsection:constraints}}\\
    \textbf{Admissibility constraints and protocol equivalence}\\[2pt]
    \begin{tabular}{@{}p{1.2em}@{}p{\dimexpr\linewidth-1.2em\relax}@{}}
        $\bullet$ & Hierarchy constraints\\
        $\bullet$ & Epistemic confidence constraint
    \end{tabular}
};

\node[branch, minimum height=44mm, anchor=north east] (tasks)
    at ([yshift=-9mm]framework.south east) {
    \textbf{Sections~\ref{section:regression}--%
    \ref{section:classification}}\\
    \textbf{Task-specific results}\\[3pt]
    \begin{tabular}{@{}p{1.2em}@{}p{\dimexpr\linewidth-1.2em\relax}@{}}
        $\bullet$ & Complementarity in regression under squared loss\\
        $\bullet$ & Impossibility theorems in classification
    \end{tabular}
};

\coordinate (bottom) at
    (framework.center |- constraints.south);

\node[wide, anchor=north] (discussion)
    at ([yshift=-10mm]bottom) {
    \textbf{Section~\ref{sec:discussion_conclusions}}\\
    \textbf{Discussion and conclusions}\\[2pt]
};

\draw[arrow] (related) -- (framework);
\draw[arrow] (framework.south) -- (equiv.north);
\draw[arrow] (framework.south) -- (tasks.north);
\draw[arrow] (equiv) -- (constraints);
\draw[arrow] (constraints.south) -- (discussion.north);
\draw[arrow] (tasks.south) -- (discussion.north);

\end{tikzpicture}

\caption{Logical structure of the manuscript.}
\label{fig:paper_roadmap}
\end{figure}

\section{Related Work on Formalizations of Complementarity}
\label{sec:related}
In research on human--AI interaction, complementarity is the case in which the \emph{human--AI team} performs better than either component alone \citep{Bansal2021CHI,Bansal2021AAAI,hemmer2021human,Hemmer2025EJIS}. 
A central formalization is \emph{complementarity team performance}, which treats complementarity as a binary property of a prediction-task human--AI interaction: complementarity is achieved whenever the empirical loss of the team is strictly lower than the minimum of the empirical losses of the human and the AI considered separately \citep{hemmer2021human,Hemmer2025EJIS}. Complementarity extends the selector logic characteristic of reliance relations, where the output is constrained to coincide with either the human or the AI prediction, by allowing interaction outputs that need not equal either input \citep{schemmer_appr_reliance,zhang2020effect}.
Existing formal work focuses on the two-agent human--AI setting. Donahue, Chouldechova, and Kenthapadi~\citep{Donahue2022FAccT} study complementarity in binary classification by partitioning the input space into loss-homogeneous \emph{regimes}, namely,  regions of cases for which human, algorithmic, and combined performance can be summarized by loss rates \citep{Donahue2022FAccT}. Their impossibility results show that, in this regime-level loss framework, complementarity cannot be achieved if human and algorithm loss rates are constant over regimes or if one agent always weakly dominates the loss of the other \citep{Donahue2022FAccT}.   Other work introduced an optimization framework in which a \emph{joint policy} is obtained by convexly combining a human policy and an ML policy at the instance level \citep{rastogi2023taxonomy}. Metrics of complementarity quantify how the optimal joint policy distributes contribution across the two agents. 
A related line of work has begun to study complementarity, deferral, and
collaboration with multiple human or machine experts
\citep{hemmer2022forming,verma2023learning,paat2025conformal,peng2025no}.
These approaches show that complementarity depends on allocation, deferral,
expert selection, and the distribution of expertise across agents. Their design problem is typically which expert or model to query, whether a case should be deferred, or how cases should be allocated among several available sources. A further adjacent literature studies forecast aggregation and opinion pooling, where multiple predictions or probabilistic forecasts are combined into a single forecast, often under proper scoring rules, calibration, regret, or optimality criteria \citep{neyman2023proper,neyman2023no}. 

Taken together, these works identify two central mechanisms in multi-agent prediction systems: selecting among available sources and pooling their predictions. However, these mechanisms still abstract away from the structure of real-world multi-agent interactions. In many HAI settings, the final prediction is not produced only by selecting one source or by applying a single pooling rule to a set of forecasts. Rather, several human and AI agents may contribute in an ordered interaction, with intermediate outputs being combined, revised, or passed on before a final judgment is produced. This limitation matters because many real-world HAI settings are not naturally two-agent interactions. In medicine, for instance, decisions may involve a general practitioner, a specialist, a nurse, a patient or caregiver, and one or more AI-based decision-support tools. In education, social services, or public administration, domain experts, case workers, lay users, and AI systems may all contribute to a final judgment. In such settings, complementarity depends on which agents participate, how they are ordered, how intermediate outputs are combined, and which benchmark is used to assess whether the resulting interaction improves on its members. We address this gap in what follows.

\section{Tree-Based Formalization of Multi-Agent Complementarity}
\label{sec:tree}
We introduce a tree-based formalism for empirical complementarity in multi-agent prediction-task HAIs. Our approach is as follows: (i) we distinguish HAIs, prediction-task HAIs, protocols, and protocol trees; (ii) we introduce notation for agents, roles, and ordered configurations; and (iii) we define tree-relative complementarity functionals.

\subsection{Preliminaries: HAIs, Agents, Roles, Ordered Configurations, and Protocols}
A prediction task is a tuple \(\tau=(\X,\Y,\Yhat,\ell)\), where \(\X\) is an input space,
\(\Y\) is the label space, \(\Yhat\) is the prediction space, and
\(\ell:\Y\times\Yhat\to[0,\infty)\), \((y,\hat y)\mapsto\ell(y,\hat y)\),
is a pointwise loss function. A labeled dataset is a finite set \(D=\{(x_i,y_i)\}_{i=1}^n\subseteq \X\times\Y.\)
We use \emph{human--AI interaction} (HAI) as a high-level term for a goal-oriented process in which human and AI agents contribute informational inputs that are integrated into a single output. In this paper we restrict attention to \emph{prediction-task HAIs}: HAIs relative to a prediction task \(\tau=(\X,\Y,\Yhat,\ell)\), in which the relevant inputs and outputs are prediction vectors and the HAI output is evaluated against ground-truth labels in a dataset \(D\). Thus, a prediction-task HAI specifies the agents---either human or AI systems---their roles, the prediction task, and the goal of producing an HAI output, but it does not yet determine how the agents interact.  We call a \emph{protocol} a concrete realization of such a prediction-task HAI: it specifies the order in which agent predictions enter the interaction, the local rules by which they are combined, and the way intermediate outputs are propagated to a final prediction. Hence, \emph{the same prediction-task HAI may admit several protocols}. We elaborate on these concepts in what follows.

\begin{definition}[Agents, roles, and ordered configurations]
\label{def:configuration}
Let $\Agents=\{a_1,\dots,a_N\}$ be a finite set of agents, let \(\Roles\) be a finite set of roles, and let $r:\Agents\to\Roles$ be a function assigning a role to each agent. An \emph{ordered agent-role configuration} is a tuple
\[
c_{\sigma,N}
=
\bigl(
(a_{\sigma(1)},r(a_{\sigma(1)})),\dots,
(a_{\sigma(N)},r(a_{\sigma(N)}))
\bigr),~ \sigma\in\mathfrak S_N,
\]
where \(\mathfrak S_N\) denotes the permutation group on \(N\) elements.
\end{definition}

For readability, in the discussion  below we suppress agent identities and display only the ordered role sequence. In what follows, we assume all agents act on the same dataset $D$ for the same prediction task \(\tau\), and all agents predict the same target variable. Thus, each agent \(a_{\sigma(i)}\) in the configuration \(c_{\sigma,N}\) is represented by a predictor $f_{a_{\sigma(i)}}:\X\to\Yhat$,
with empirical prediction vector
\[
\hat y^{(\sigma(i))}=(f_{a_{\sigma(i)}}(x_1),\dots,f_{a_{\sigma(i)}}(x_n))\in\Yhat^n.
\]
Once \(\tau\) and \(D\) are fixed, we use the ordered configuration \(\mathbf c_{\sigma,N}\) as the combinatorial representation of the HAI agents and roles. This representation deliberately abstracts away from human-specific cognitive or behavioral mechanisms: human and AI agents enter the formalism through their roles and prediction vectors. This abstraction is useful for isolating benchmark, aggregation, and topology effects when coupled with the representation of interaction protocols with rooted planar binary trees.

\begin{definition}[Rooted planar binary trees]
\label{def:pbt_recursive}
Let \(\Tcal_N\) denote the set of rooted planar binary trees with \(N\) leaves,
equivalently \(N-1\) internal vertices. We use the recursive convention
$\Tcal_1=\{\,|\,\}$, where \(|\) is the one-leaf tree, with no internal vertex. For \(N\ge2\), define
\[
\Tcal_N
=
\coprod_{\substack{p+q=N\\ p,q\ge1}}
\left\{
\sigma\vee\tau:
\sigma\in\Tcal_p,\ \tau\in\Tcal_q
\right\}.
\]
Here \(\sigma\vee\tau\) denotes root grafting: a new internal root vertex is
created whose left subtree is \(\sigma\) and whose right subtree is \(\tau\).
For \(\mathsf T\in\Tcal_N\), let \(V(\mathsf T)\) denote its set of internal
vertices. Then \(|V(\mathsf T)|=N-1\). We define the left-to-right ordering $o(\mathsf T)$ of \(V(\mathsf T)\) recursively. For the one-leaf tree, \(o(|)=\varnothing\).
If $\mathsf T=\sigma\vee_x\tau$, where \(x\) denotes the newly created root vertex, then
\(
o(\mathsf T)
=
\{\,o(\sigma)<x<o(\tau)\,\}.
\)
\end{definition}
The cardinality of \(\Tcal_N\) is the Catalan number
$
|\Tcal_N|=C_{N-1}=\frac{1}{N}\binom{2N-2}{N-1}.$

\begin{definition}[Protocol trees for a configuration]
\label{def:decorated_tree}
Let \(c_{\sigma,N}\) be an ordered agent-role configuration of length
\(N\). A \emph{protocol tree for \(\mathbf c_{\sigma,N}\)} is a rooted planar binary tree \(\mathsf T\in\Tcal_N\) whose leaves are decorated from left to right by \(c_{\sigma,N}\). Equivalently, once the predictors induced by \(c_{\sigma,N}\) are fixed, the leaves of \(\mathsf T\) are labeled from left to right by the corresponding ordered
prediction vectors $\hat y^{(\sigma(1))},\dots,\hat y^{(\sigma(N))}$.
\end{definition}

\begin{remark}[Examples of configurations]
\label{rem:configuration_examples}
Configurations with the same multiset of roles may still represent different HAIs because the order of the leaves is workflow-relevant. For \(N\!=\!2\), let $A\!=\!\{a_1,a_2\}$ and $r(a_1)\!=\!\textnormal{\texttt{expert}}$, $r(a_2)\!=\!\textnormal{\texttt{assistant}}$ (both agents are human). Then $c_{\sigma_1,2}=(\textnormal{\texttt{expert}},\textnormal{\texttt{assistant}}), c_{\sigma_2,2}\!=\!(\textnormal{\texttt{assistant}},\textnormal{\texttt{expert}})$ are distinct configurations. For \(N\!=\!3\), let $A\!=\!\{a_1,a_2,a_3\}$ and $r(a_1)\!=\!\textnormal{\texttt{expert}}$, $r(a_2)\!=\!\textnormal{\texttt{assistant}}$, and  $r(a_3)\!=\!\textnormal{\texttt{AI system}}$---see Figure~\ref{fig:N3_trees_horizontal}. Examples such as
\(
c_{\sigma_1,3}=(\textnormal{\texttt{expert}},\textnormal{\texttt{assistant}}, \textnormal{\texttt{AI}}), \quad c_{\sigma_2,3}=(\textnormal{\texttt{assistant}},\textnormal{\texttt{AI}},\textnormal{\texttt{expert}})\)
encode different interaction orders and therefore different HAI protocols.  
\end{remark}

Throughout the paper, the ordered configuration \(c_{\sigma,N}\) is treated as fixed. Thus, we do not consider permutations of the leaves and we work with rooted planar binary trees because left-to-right leaf order is part of the workflow specification; in particular, mirror-related trees are not identified. For notational simplicity, once a configuration \(c_{\sigma,N}\) is specified, we write the ordered leaf predictions simply as
\(\hat y^{(1)},\dots,\hat y^{(N)}\in\Yhat^n\).

\subsection{Local Compositions}
\label{subsection:local_interactions}
We consider two types of local aggregation rules of the in-
ternal nodes of protocol trees.\ \emph{Selectors} are set-theoretic local binary compositions that return one of their inputs coordinatewise on ordered sets. Examples are $m_2(u,v)=u$, $m_2(u,v)=v$, $m_2(u,v)=\min\{u,v\}$, and $m_2(u,v)=\max\{u,v\}$. Let $\rho:I\to\R$ be continuous and strictly monotone increasing, where $I\subseteq\R$ is an interval. For \(\alpha\in[0,1]\), \emph{quasi-arithmetic means} $m_{2,\alpha}^{\rho}$ \citep{andrey1930notion,nagumo1930klasse} act coordinatewise as follows: 
\begin{equation}
(m_{2,\alpha}^{\rho}(u,v))_i
:=
\rho^{-1}\!\bigl(
\alpha\rho(u_i)+(1-\alpha)\rho(v_i)
\bigr),
\label{eq:quasi_arithmetic_mean}
\end{equation}
for \(u,v\in I^n\). If \(\rho=\mathrm{id}\), one obtains the family of \emph{linear pooling} functions
\(\alpha u+(1-\alpha)v\). If \(I=(0,1)\) and \(\rho=\operatorname{logit}\), that is $\rho(p)=\log \left(\frac{p}{1-p}\right), p\in(0,1)$, one obtains \emph{logit pooling} functions.   While the four examples of selector compositions are associative, quasi-arithmetic means are non-associative except in the projection cases: 

\begin{lemma}
\label{lem:quasi_assoc}
$m_{2,\alpha}^{\rho}$ is associative if and only if  $\alpha\in\{0,1\}$. 
\end{lemma}
\begin{proof}
Let \(I\) contain at least two points. For all $u,v,w\in I$, \(m^\rho_{2,\alpha}(m^\rho_{2,\alpha}(u,v),w)=m^\rho_{2,\alpha}(u,m^\rho_{2,\alpha}(v,w))\) is equivalent to \(\alpha^2=\alpha\), \((1-\alpha)=(1-\alpha)^2\). Then $\alpha\in\{0,1\}$.
\end{proof}
Lemma~\ref{lem:quasi_assoc} shows that every nontrivial weight choice \(\alpha\in(0,1)\) yields a non-associative local rule $m_{2,\alpha}^{\rho}$. Hence, with quasi-arithmetic means, tree topology becomes a potential source of variation for complementarity.

\subsection{Two-Agent Complementarity Functional}
\label{subsection:complementarity}
Let a human agent and an AI system interact in a HAI. Let $\hat y^H$, $
\hat y^{AI}$ denote the vectors of predictions of the human agent and the AI system  over the dataset $D$. Let $
\hat y^{HAI}$ denote the vector of `human--AI' team predictions that results from the interaction between the human and the AI. These predictions depend on $\hat y^H$ and $\hat y^{AI}$. Consider the average empirical losses
\(L_{\smallclubsuit}(D):=
\frac1n\sum_{i=1}^n \ell(y_i,\hat y_i^{\smallclubsuit})\), where 
\(\clubsuit\in\{H,AI,HAI\}\).

\begin{definition}[Two-agent complementarity \citep{Bansal2021CHI,Donahue2022FAccT,ferrario2026epistemology}]
\label{def:classical_ctp} 
Define
\begin{align}
\Psi (\hat y^{H},\hat y^{AI};D):=\min\{L_H(D),L_{AI}(D)\} -L_{HAI}(D).\label{eq:2_complem}
\end{align}
The human--AI team achieves complementarity if 
\(
\Psi(\hat y^{H},\hat y^{AI};D)>0.
\)
\end{definition}
Thus, complementarity records relative empirical gain between a human agent and an AI system by comparing the \emph{average losses} of the `human--AI team' with those of the human and the AI system independently, on a given dataset. As with supervised-learning metrics such as accuracy,
\emph{complementarity is an ex post empirical evaluation criterion} because it depends on the ground-truth labels in \(D\). It typically cannot
assess, during an ongoing decision process, whether the \emph{current} interaction exhibits
complementarity; this can only be measured retrospectively once outcomes are
known.

\subsection{Multi-Agent Complementarity Functionals}
\label{subsection:multi_agent_compl}
We now move to the complementarity functional for \(N\ge 2\) agents. We start with a structural assumption:

\begin{assumption}[Benchmark--interaction decomposition]
\label{ass:benchmark}
Complementarity functionals have the form
\begin{align*}
\Psi(\hat y^{(1)},\dots,\hat y^{(N)}; D)=\Phi(\hat y^{(1)},\dots,\hat y^{(N)}; D)-
\Theta(\hat y^{(1)},\dots,\hat y^{(N)}; D),
\end{align*}
where \(\Phi\) is a benchmark term independent of protocol topology, and \(\Theta\) is the loss of the protocol output.
\end{assumption}
This assumption generalizes the structure of the complementarity functional in Definition~\ref{def:classical_ctp} to the multi-agent setting.  Then, the problem becomes to (1) introduce the interaction term $\Theta$, and (2) define possible benchmark terms $\Phi$.

\subsubsection{The interaction term $\Theta$} Multi-agent complementarity workflows must cover sequential workflows, such as \texttt{expert}--\texttt{assistant}--\texttt{AI} configurations, deferred referral structures, and settings in which several \texttt{human}--\texttt{AI} dyads are combined into a final decision. The interaction term $\Theta$ separates local interaction from global protocol structure. The same local rule can be evaluated on different trees, allowing us to ask whether complementarity depends on the local aggregation mechanism or on the order in which agents enter the workflow. Finally, although the binary restriction is a simplification as some workflows may involve simultaneous or higher-arity interactions, it yields a mathematically tractable class of protocols: binary trees expose questions about associativity, attainable prediction regions, and optimizing complementarity under a chosen loss. Thus, we arrive at:

\begin{assumption}[Local binary composition]
\label{ass:local}
Let \(\mathsf{T}\in\Tcal_N\) with leaves labeled by $\hat y^{(1)},\dots,\hat y^{(N)}\in\Yhat^n$. Its output 
\(
\hat y_{\mathsf T}
:=
m_{\mathsf T}(\hat y^{(1)},\dots,\hat y^{(N)})\in\Yhat^n
\)
is defined recursively by applying a local binary rule
$m_2:\Yhat^n\times\Yhat^n\to\Yhat^n$ at each internal node of \(\mathsf T\).
\end{assumption}

\subsubsection{The benchmark term $\Phi$} Two natural generalizations of the benchmark term in Definition~\ref{def:classical_ctp} are possible: 
\begin{align}
&\Phi_{\mathrm{ba}}(\hat y^{(1)},\dots,\hat y^{(N)}; D):= 
\min_{1\le j\le N}\frac1n\sum_{i=1}^n
\ell\bigl(y_i,\hat y_i^{(j)}\bigr),\label{eq:benchmark_ba}\\
&\Phi_{\mathrm{pw}}(\hat y^{(1)},\dots,\hat y^{(N)}; D):=
\frac1n\sum_{i=1}^n
\min_{1\le j\le N}
\ell\bigl(y_i,\hat y_i^{(j)}\bigr). \label{eq:benchmark_pw}
\end{align}
These benchmarks answer different questions. The \emph{best average} benchmark~\eqref{eq:benchmark_ba} asks whether the interaction improves over deploying the best fixed component across all cases. It generalizes the benchmark in~\eqref{eq:2_complem} from the literature on complementarity \citep{Donahue2022FAccT,ferrario2026epistemology} to \(N\) agents. The \emph{pointwise-min} benchmark~\eqref{eq:benchmark_pw} asks whether the protocol improves over the best available prediction at the case level, before averaging, instead. This distinction is relevant for HAI research because the best average benchmark can report complementarity even when the interaction fails to improve over the best available prediction on each case. Consider two agents, \(H\) and \(A\), and losses in arbitrary units:
\begin{equation*}
\begin{array}{c|ccc}
 & \ell(y_i,\hat y_i^H) & \ell(y_i,\hat y_i^{AI}) & \ell(y_i,\hat y_i^{HAI}) \\
\hline
\text{case }1 & 0 & 4 & 1 \\
\text{case }2 & 4 & 0 & 1
\end{array}
\end{equation*}
Hence, using the best average benchmark \(\Psi_{\mathrm{ba}}=2-1=1>0\), while \(\Psi_{\mathrm{pw}}=0-1=-1<0\) using the pointwise-min one. Thus, only the best average benchmark reports complementarity. 
For this reason, \emph{the pointwise-min benchmark is useful in workflows where a more demanding aggregate criterion is required}. Examples include medical diagnosis, triage, and review, judicial or administrative decision support, and safety-critical model monitoring. The same consideration applies to multi-AI and AI-agent environments, where several models, model versions, systems, or agents may produce predictions for the same case. In such settings, a best average benchmark may show that an ensemble or multi-agent protocol improves over the best fixed model on average, while still failing to use the best available prediction on important cases. In lower-stakes or large-scale settings where the primary object of evaluation is average performance, the best average benchmark may be sufficient.

\subsubsection{Tree-relative multi-agent complementarity functional} We introduce the complementarity functional for interactions with $N\ge 2$ agents using our tree-based formulation. It is the main mathematical object of this work.

\begin{definition}[Tree-relative complementarity functional]
\label{def:tree_functional}
Let \(\tau=(\X,\Y,\Yhat,\ell)\) be a prediction task and $D=\{(x_i,y_i)\}_{i=1}^n$ be a labeled dataset. Let $m_2$ be a local binary composition rule and $\hat y^{(1)},\dots, \hat y^{(N)}$ be the $N$ predictions of $N$ agents of a configuration $c_{\sigma, N}$. For any protocol tree $\mathsf T\in\Tcal_N$ for $c_{\sigma, N}$, the tree-relative complementarity functional $\Psi_{\mathsf{T},\diamondsuit}^{m_{\mathsf{T}}}$ is defined as 
\begin{equation}
\Psi_{\mathsf{T},\diamondsuit}^{m_{\mathsf T}}(\hat y^{(1)},\dots,\hat y^{(N)};D):=
\Phi_\diamondsuit(\hat y^{(1)},\dots,\hat y^{(N)};D)
-\Theta_\mathsf{T}(\hat y^{(1)},\dots,\hat y^{(N)};D),
\label{eq:compl_functional}
\end{equation}
where \(\diamondsuit\in\{\mathrm{ba},\mathrm{pw}\}\) and the interaction term reads
\begin{equation}
\Theta_\mathsf{T}(\hat y^{(1)},\dots,\hat y^{(N)};D):=\frac1n\sum_{i=1}^n
\ell\bigl(y_i,
\hat y_i^{\mathsf T}
\bigr).
\label{eq:tree_functional}
\end{equation}
\end{definition}

The following result is straightforward.

\begin{lemma}
Let \(\Psi_{\mathsf{T},\diamondsuit}^{m_{\mathsf T}}\) be as in~\eqref{eq:compl_functional}. For all predictions \((\hat y^{(1)},\dots,\hat y^{(N)})\) and datasets \(D\),
\[
\Psi_{\mathsf{T},\mathrm{pw}}^{m_{\mathsf T}}
(\hat y^{(1)},\dots,\hat y^{(N)};D)
\le \Psi_{\mathsf{T},\mathrm{ba}}^{m_{\mathsf T}}(\hat y^{(1)},\dots,\hat y^{(N)};D)
\]
The two complementarity functionals coincide if and only if at least one fixed agent is pointwise loss-minimal on every sample in \(D\). 
\end{lemma}
\begin{proof}
Use \(\frac{1}{n}\sum_{i=1}^n\min_j\{a^{(j)}_i\}\le \min_{j}\{\frac{1}{n}\sum_{i=1}^n a_i^{(j)}\}\) for all $a_i^{(j)}\in\R$.
\end{proof}

We call
\begin{equation}
\Delta_{\mathrm{ba,pw}}:=\Phi_{\mathrm{ba}}(\hat y^{(1)},\dots,\hat y^{(N)};D)-\Phi_{\mathrm{pw}}(\hat y^{(1)},\dots,\hat y^{(N)};D)\ge 0 \label{eq:gap}
\end{equation}
the \emph{benchmark gap}. Since the interaction term is the same for both benchmarks,
\[
\Psi_{\mathsf T,\mathrm{ba}}^{m_{\mathsf T}}
=
\Psi_{\mathsf T,\mathrm{pw}}^{m_{\mathsf T}}
+
\Delta_{\mathrm{ba,pw}}.
\]

In the remainder, \emph{we keep the benchmark explicit when it matters}. Results on regression optimization, protocol-indifference loci, barycentric coordinates, and Tamari reparameterizations are benchmark-independent up to an additive constant---see Sections~\ref{section:linear_protocol_coordinates}
and~\ref{section:regression}. By contrast, selector and classification impossibility results hold for the pointwise-min benchmark; under the best-average benchmark they become upper bounds controlled by the benchmark gap $\Delta_{\mathrm{ba,pw}}\ge 0$---see Section~\ref{section:classification}.

\section{Protocol Equivalence and Tamari Reparameterizations
Under Linear Pooling}
\label{section:linear_protocol_coordinates}
We study the structure of multi-agent HAI protocols under
linear pooling, independently of the prediction task and loss
function. We establish two related results.

\begin{enumerate}[leftmargin=*,nosep]
\item First, every protocol tree provides a bijective parametrization of the same simplex of global leaf weights. Consequently, any two
trees with the same ordered leaves admit an output-preserving
reparameterization (Section~\ref{subsection:barycentric_coordinates}). 
\item Elementary Tamari rotations provide explicit local transformations implementing these reparameterizations, whose compositions are independent of the chosen directed path (Section~\ref{subsection:tamari}).
\end{enumerate}

Figure~\ref{fig:barycentric_tamari} illustrates the relationship between local parameters and global leaf weights.  Throughout this section, assume that the prediction space $\Yhat$ is a convex subset of a real vector space, and that every internal node uses the linear-pooling rule
\(m_{2,\alpha}(u,v)=\alpha u+(1-\alpha)v,\  \alpha\in(0,1)\). We consider unrestricted choices of these local weights, with the ordered leaf predictions held fixed.

\begin{figure*}[h!]
\centering
\begin{tikzpicture}[>=Latex, font=\small, line cap=round, line join=round]

\tikzset{
    edge/.style={draw, thick},
    hidden/.style={draw, thin, dotted},
    axisarrow/.style={->, thick},
    maparrow/.style={->, thin},
    transarrow/.style={->, thick, dashed},
    pt/.style={circle, fill=black, inner sep=1.5pt},
    lab/.style={font=\scriptsize},
    toplab/.style={font=\small},
    treeedge/.style={line width=0.5pt, line cap=round},
    inode/.style={circle, fill=black, inner sep=0pt},
    leaf/.style={circle, fill=black, inner sep=0pt},
    treelabel/.style={font=\scriptsize},
    vertexcircle/.style={draw=black, fill=white, line width=0.5pt}
}

\def\Rcmini{0.56}

\coordinate (O) at (0,0);
\coordinate (ex) at (3.2,0);
\coordinate (ey) at (1.15,0.78);   
\coordinate (ez) at (0,3.2);

\coordinate (X)   at ($(O)+(ex)$);
\coordinate (Y)   at ($(O)+(ey)$);
\coordinate (Z)   at ($(O)+(ez)$);
\coordinate (XY)  at ($(X)+(ey)$);
\coordinate (XZ)  at ($(X)+(ez)$);
\coordinate (YZ)  at ($(Y)+(ez)$);
\coordinate (XYZ) at ($(XY)+(ez)$);

\draw[axisarrow] (O) -- ($(X)+(0.85,0)$);
\draw[axisarrow] (O) -- ($(Y)+(0.65,0.44)$);
\draw[axisarrow] (O) -- ($(Z)+(0,0.85)$);

\node[lab, below] at ($(X)+(0.90,0)$) {$\alpha_1$};
\node[lab, above] at ($(Y)+(0.68,0.46)$) {$\alpha_2$};
\node[lab, left]  at ($(Z)+(0,0.88)$) {$\alpha_3$};

\draw[edge] (O)--(X)--(XZ)--(Z)--cycle;     
\draw[edge] (YZ)--(Z);                      
\draw[edge] (X)--(XY)--(XYZ)--(XZ);         
\draw[edge] (YZ)--(XYZ);

\draw[hidden] (Y)--(XY);
\draw[hidden] (Y)--(YZ);

\node[toplab] at ($(O)!0.5!(XYZ)+(0,2.5)$)
{Local parameter space \((0,1)^3\)};

\node[lab, below left] at (O) {\((0,0,0)\)};
\node[lab, below] at (X) {\((1,0,0)\)};

\coordinate (pa) at ($(O)+0.16*(ex)+0.14*(ey)+0.18*(ez)$);
\coordinate (pb) at ($(O)+0.8*(ex)+0.1*(ey)+0.83*(ez)$);
\coordinate (pc) at ($(O)+0.24*(ex)+0.66*(ey)+0.57*(ez)$);

\node[pt, label={[lab]below left:\(\boldsymbol{\alpha}\)}] at (pa) {};
\node[pt, label={[lab]above left:\(\boldsymbol{\alpha}'\)}] at (pb) {};
\node[pt, label={[lab]above right:\(\boldsymbol{\alpha}^{\prime\prime}\)}] at (pc) {};


\begin{scope}[shift={(0.00,1.05)}]
    \draw[vertexcircle] (0,0) circle (\Rcmini);
    \begin{scope}[scale=0.2, transform shape]
        \node[leaf] at (-1.8,1.25) {};
        \node[leaf] at (-0.6,1.25) {};
        \node[leaf] at (0.6,1.25) {};
        \node[leaf] at (1.8,1.25) {};

        \node[inode] at (-1.2,0.55) {};
        \node[inode] at (-0.6,-0.15) {};
        \node[inode] at (0,-0.85) {};

        \draw[treeedge] (-1.8,1.25) -- (-1.2,0.55);
        \draw[treeedge] (-0.6,1.25) -- (-1.2,0.55);
        \draw[treeedge] (-1.2,0.55) -- (-0.6,-0.15);
        \draw[treeedge] (0.6,1.25) -- (-0.6,-0.15);
        \draw[treeedge] (-0.6,-0.15) -- (0,-0.85);
        \draw[treeedge] (1.8,1.25) -- (0,-0.85);
    \end{scope}

    \node[treelabel] at (0,-0.40) {$\mathsf T$};
\end{scope}

\begin{scope}[shift={(3.02,3.50)}]
    \draw[vertexcircle] (0,0) circle (\Rcmini);
    \begin{scope}[scale=0.2, transform shape]
        \node[leaf] at (-1.8,1.25) {};
        \node[leaf] at (-0.6,1.25) {};
        \node[leaf] at (0.6,1.25) {};
        \node[leaf] at (1.8,1.25) {};

        \node[inode] at (-1.2,0.55) {};
        \node[inode] at (-0.6,-0.15) {};
        \node[inode] at (0,-0.85) {};

        \draw[treeedge] (-1.8,1.25) -- (-1.2,0.55);
        \draw[treeedge] (-0.6,1.25) -- (-1.2,0.55);
        \draw[treeedge] (-1.2,0.55) -- (-0.6,-0.15);
        \draw[treeedge] (0.6,1.25) -- (-0.6,-0.15);
        \draw[treeedge] (-0.6,-0.15) -- (0,-0.85);
        \draw[treeedge] (1.8,1.25) -- (0,-0.85);
    \end{scope}

    \node[treelabel] at (0,-0.40) {$\mathsf T$};
\end{scope}

\begin{scope}[shift={(0.85,2.80)}]
    \draw[vertexcircle] (0,0) circle (\Rcmini);
    \begin{scope}[scale=0.2, transform shape]
        \node[leaf] at (-1.8,1.25) {};
        \node[leaf] at (-0.6,1.25) {};
        \node[leaf] at (0.6,1.25) {};
        \node[leaf] at (1.8,1.25) {};

        \node[inode] at (0,0.55) {};
        \node[inode] at (-0.6,-0.15) {};
        \node[inode] at (0,-0.85) {};

        \draw[treeedge] (-0.6,1.25) -- (0,0.55);
        \draw[treeedge] (0.6,1.25) -- (0,0.55);
        \draw[treeedge] (-1.8,1.25) -- (-0.6,-0.15);
        \draw[treeedge] (0,0.55) -- (-0.6,-0.15);
        \draw[treeedge] (-0.6,-0.15) -- (0,-0.85);
        \draw[treeedge] (1.8,1.25) -- (0,-0.85);
    \end{scope}

    \node[treelabel] at (0,-0.40) {$\mathsf T'$};
\end{scope}

\draw[transarrow] (pa)--(pc)
node[midway, above left, lab]
{\(\Gamma_{\mathsf T,\mathsf T'}\)};

\begin{scope}[shift={(7.8,0.25)}]

\def\S{3.5}

\coordinate (S1) at (0,0);
\coordinate (S2) at (\S,0);
\coordinate (S4) at ({0.5*\S},{0.8660254*\S});

\coordinate (S3) at ({0.5*\S},{0.38*\S});

\draw[edge] (S1)--(S2)--(S4)--cycle;

\draw[hidden] (S3)--(S1);
\draw[hidden] (S3)--(S2);
\draw[hidden] (S3)--(S4);

\node[toplab] at ($(S1)!0.5!(S4)+(0.6,2.7)$)
{Simplex of global leaf weights \(\Delta^3\)};


\coordinate (W1) at ({0.82*\S},{0.18*\S});
\coordinate (W2) at ({0.45*\S},{0.60*\S});

\draw[maparrow] (pa) .. controls (-2.7,1.2) .. (W1);
\draw[maparrow] (pb) .. controls (-2.5,2.5) .. (W2);
\draw[maparrow] (pc) .. controls (-1.8,2.0) .. (W1);

\node[
    pt,
    label={[lab]below left:
    \(\varphi_{\mathsf T}(\boldsymbol{\alpha})
    =\varphi_{\mathsf T'}(\boldsymbol{\alpha}^{\prime\prime})\)}
] at (W1) {};

\node[
    pt,
    label={[lab]below:
    \(\varphi_{\mathsf T}(\boldsymbol{\alpha}')\)}
] at (W2) {};

\end{scope}

\end{tikzpicture}

\caption{Under linear pooling \(m_{2,\alpha}\), a protocol tree maps local parameters from the open unit cube \((0,1)^3\) to a point of the interior of the simplex \(\Delta^3\). Distinct parameter vectors (\(\boldsymbol{\alpha}\), \(\boldsymbol{\alpha}'\)) of the same tree $\mathsf T=(((12)3)4)$ yield different barycentric weights \(\varphi_{\mathsf T}(\boldsymbol{\alpha})\) and \(\varphi_{\mathsf T}(\boldsymbol{\alpha}')\); these are the global leaf weights determining the corresponding linear
combination of the fixed leaf predictions. A Tamari reparameterization \(\Gamma_{\mathsf T,\mathsf T'}\) transports \(\boldsymbol{\alpha}\) to \(\boldsymbol{\alpha}^{\prime\prime}\) so that the rotated tree $\mathsf T'=((1(23))4)$ represents the same simplex point \(\varphi_{\mathsf T}(\boldsymbol{\alpha})=\varphi_{\mathsf T'}(\boldsymbol{\alpha}'')\), and hence the same protocol output \(\hat{y}^{\mathsf T}(\boldsymbol{\alpha})=\hat{y}^{\mathsf T'}(\boldsymbol{\alpha}'')\) and complementarity value.}
\label{fig:barycentric_tamari}
\end{figure*}

\subsection{Protocol-Tree Coordinate Systems for Global Leaf Weights}
\label{subsection:barycentric_coordinates}

Let 
\(
\Delta^{N-1}
=
\left\{
(\omega_1,\dots,\omega_N)\in[0,1]^N:
\sum_{j=1}^N\omega_j=1
\right\}
\)
be the standard $(N-1)$-simplex. 

\begin{definition}[Barycentric coordinate map of a linear protocol tree]
\label{def:barycentric_coordinate_map}
Let \(\mathsf T\in\Tcal_N\). The \emph{barycentric coordinate map} associated
with \(\mathsf T\) is the map $\varphi_{\mathsf T}:[0,1]^{N-1}\to\Delta^{N-1}$,
defined recursively as follows. For the one-leaf tree \(|\in\Tcal_1\), set
$\varphi_{{|}}(\emptyset)=(1)\in\Delta^0$.
For \(N\ge2\), let $\mathsf T=\sigma\vee\tau$,
$\sigma\in\Tcal_{N_1}$, $\tau\in\Tcal_{N_2}$, $N_1+N_2=N$, be the unique root decomposition of \(\mathsf T\). With the
recursive left-to-right order $o(\mathsf T)=\{o(\sigma)<x<o(\tau)\}$ of internal nodes, for
$\boldsymbol{\alpha}_{N-1}$, the parameters
\((\alpha_1,\dots,\alpha_{N_1-1})\) belong to \(\sigma\), the parameter
\(\alpha_{N_1}\) belongs to the root, and the parameters
\((\alpha_{N_1+1},\dots,\alpha_{N-1})\) belong to \(\tau\). We define
\begin{align}
\varphi_{\mathsf T}(\boldsymbol\alpha)
:=
\Bigl(
&\alpha_{N_1}\,
\varphi_\sigma(\alpha_1,\dots,\alpha_{N_1-1}),(1-\alpha_{N_1})\,
\varphi_\tau(\alpha_{N_1+1},\dots,\alpha_{N-1})
\Bigr).
\end{align}
\end{definition}
Under linear pooling, for fixed leaf predictions \(\hat y^{(j)}\), changing the local parameters \(\boldsymbol\alpha\) changes the protocol output \(\hat y^{\mathsf T}(\boldsymbol\alpha)\) by changing the coefficients of the convex combination
\[
\hat y^{\mathsf T}(\boldsymbol\alpha)
=
\sum_{j=1}^N
\varphi_{\mathsf T,j}(\boldsymbol\alpha)\hat y^{(j)},
\qquad
\varphi_{\mathsf T}(\boldsymbol\alpha)\in\Delta^{N-1}.
\]
The map \(\varphi_{\mathsf T}\) provides a coordinate system for the
global leaf weights: it translates the tree parameters
\(\boldsymbol\alpha\) into the coefficients
\(\varphi_{\mathsf T,j}(\boldsymbol\alpha)\) assigned to the fixed leaf
predictions. These weights determine the protocol output through the
displayed convex combination.

\begin{proposition}\label{prop:barycentric}
For every \(\mathsf T\in\Tcal_N\), the restriction
\[
\varphi_{\mathsf T}\big|_{(0,1)^{N-1}}:
(0,1)^{N-1}\longrightarrow
\operatorname{int}(\Delta^{N-1})
\]
is a bijection.
\end{proposition}
\begin{proof}
By induction on \(N\), using the recursive decomposition of
\(\mathsf T\). Let
\[
\mathsf T=\sigma\vee\tau\in\Tcal_N,
\qquad
\sigma\in\Tcal_{N_1},
\qquad
\tau\in\Tcal_{N_2},
\qquad
N_1+N_2=N.
\]
With respect to the recursive left-to-right ordering of internal nodes, the root of \(\mathsf T\) has index \(N_1\). Hence every parameter vector \(\boldsymbol\alpha\in(0,1)^{N-1}\) decomposes as
$\boldsymbol\alpha=(\boldsymbol\alpha_L,\alpha_{N_1},\boldsymbol\alpha_R)$,
where $\boldsymbol\alpha_L=(\alpha_1,\dots,\alpha_{N_1-1})$, $\boldsymbol\alpha_R=(\alpha_{N_1+1},\dots,\alpha_{N-1})$. By definition,
\[
\varphi_{\mathsf T}(\boldsymbol\alpha)
=
\bigl(
\alpha_{N_1}\varphi_\sigma(\boldsymbol\alpha_L),
(1-\alpha_{N_1})\varphi_\tau(\boldsymbol\alpha_R)
\bigr).
\]

Injectivity of \(\varphi_{\mathsf T}\) is equivalent to
\begin{equation}
\varphi_{\mathsf T}(\boldsymbol\alpha)
=
\varphi_{\mathsf T}(\boldsymbol\beta)
\quad\Longrightarrow\quad
\boldsymbol\alpha=\boldsymbol\beta.
\label{eq:inj_barycentric}
\end{equation}
We prove this by induction on \(N\). For \(N=1\), the claim is trivial, since \(\varphi(|;\emptyset)=(1)\). For \(N=2\), the unique tree satisfies
\(\varphi_{\mathsf T}(\alpha)=(\alpha,1-\alpha)\), so
\eqref{eq:inj_barycentric} holds. Assume injectivity holds for all trees with fewer than \(N\) leaves, and let
\(\mathsf T=\sigma\vee\tau\in\Tcal_N\). Let
\[
\boldsymbol\alpha
=
(\boldsymbol\alpha_L,\alpha_{N_1},\boldsymbol\alpha_R),
\qquad
\boldsymbol\beta
=
(\boldsymbol\beta_L,\beta_{N_1},\boldsymbol\beta_R),
\]
and suppose
\(\varphi_{\mathsf T}(\boldsymbol\alpha)=
\varphi_{\mathsf T}(\boldsymbol\beta)\). Equating the first \(N_1\) components
and the last \(N_2\) components gives
\begin{align*}
\alpha_{N_1}\varphi_{\sigma,i}(\boldsymbol\alpha_L)
&=
\beta_{N_1}\varphi_{\sigma,i}(\boldsymbol\beta_L),
&& i=1,\dots,N_1,\\
(1-\alpha_{N_1})\varphi_{\tau,j}(\boldsymbol\alpha_R)
&=
(1-\beta_{N_1})\varphi_{\tau,j}(\boldsymbol\beta_R),
&& j=1,\dots,N_2.
\end{align*}
Summing the first line over \(i=1,\dots,N_1\) and using
\(\sum_i\varphi_{\sigma,i}=1\) gives
\(\alpha_{N_1}=\beta_{N_1}\). Since \(\alpha_{N_1}\in(0,1)\), the first line
then reduces to $\varphi_\sigma(\boldsymbol\alpha_L)=
\varphi_\sigma(\boldsymbol\beta_L)$. By the induction hypothesis applied to \(\sigma\), we obtain
\(\boldsymbol\alpha_L=\boldsymbol\beta_L\). Similarly,
\(1-\alpha_{N_1}=1-\beta_{N_1}>0\), so the second line reduces to $\varphi_\tau(\boldsymbol\alpha_R)=\varphi_\tau(\boldsymbol\beta_R)$. By the induction hypothesis applied to \(\tau\), we obtain \(\boldsymbol\alpha_R=\boldsymbol\beta_R\). Hence
\(\boldsymbol\alpha=\boldsymbol\beta\), proving injectivity.

Surjectivity of \(\varphi_{\mathsf T}\) is equivalent to
\begin{equation}
\forall\,\boldsymbol\omega\in\operatorname{int}(\Delta^{N-1}),
\quad
\exists\,\boldsymbol\alpha\in(0,1)^{N-1}
\quad\text{such that}\quad
\varphi_{\mathsf T}(\boldsymbol\alpha)=\boldsymbol\omega.
\label{eq:surj_barycentric}
\end{equation}
We again argue by induction on \(N\). For \(N=1\), the statement is trivial. For \(N=2\), it follows from
\(\varphi_{\mathsf T}(\alpha)=(\alpha,1-\alpha)\). Assume surjectivity holds for all trees with fewer than \(N\) leaves, and let
\(\mathsf T=\sigma\vee\tau\in\Tcal_N\). Then \eqref{eq:surj_barycentric} is equivalent to  
\begin{align*}
\omega_i&=\alpha_{N_1} \varphi_{\sigma,i}(\boldsymbol{\alpha}_{L}),&& i=1,\dots, N_1 \\
\omega_{N_1+j}&=(1-\alpha_{N_1}) \varphi_{\tau,j}(\boldsymbol{\alpha}_{R}),&& j=1,\dots, N_2 
\end{align*}
Summing over $i=1,\dots, N_1$ and $j=1,\dots, N_2$ gives
\[
\alpha_{N_1}=\sum_{i=1}^{N_1}\omega_i,\quad 1-\alpha_{N_1}=\sum_{j=1}^{N_2}\omega_{N_1+j}.
\]
Then 
\begin{align*}
&\varphi_{\sigma,i}(\boldsymbol{\alpha}_{L})=\frac{\omega_i}{\sum_{l=1}^{N_1}\omega_l},\quad i=1,\dots, N_1 \\
&\varphi_{\tau,j}(\boldsymbol{\alpha}_{R})=\frac{\omega_{N_1+j}}{\sum_{s=1}^{N_2}\omega_{N_1+s}},\quad j=1,\dots, N_2.
\end{align*}
The induction step gives $\boldsymbol{\alpha}_{L}$ and $\boldsymbol{\alpha}_{R}$. This concludes the proof. 
\end{proof}


\begin{corollary}[Universal protocol equivalence]
\label{cor:universal_protocol_equivalence}
Let $\mathsf T,\mathsf U\in\Tcal_N$ be protocol trees
with the same ordered leaf predictions. Then
\[
\Gamma_{\mathsf T,\mathsf U}
:=
\varphi_{\mathsf U}^{-1}\circ\varphi_{\mathsf T}
:
(0,1)^{N-1}\longrightarrow(0,1)^{N-1}
\]
is a bijection satisfying
\[
\hat y^{\mathsf U}
\bigl(\Gamma_{\mathsf T,\mathsf U}(\boldsymbol\alpha)\bigr)=\hat y^{\mathsf T}(\boldsymbol\alpha)
\]
for every $\boldsymbol\alpha\in(0,1)^{N-1}$ and every
choice of leaf predictions. Consequently, the two trees
have identical attainable output sets and complementarity
values under any common pointwise loss and benchmark.
\end{corollary}

\begin{proof}
By Proposition~\ref{prop:barycentric}, both $\varphi_{\mathsf T}$ and $\varphi_{\mathsf U}$ are bijections onto $\operatorname{int}(\Delta^{N-1})$. Hence, their composition $\Gamma_{\mathsf T,\mathsf U}
=\varphi_{\mathsf U}^{-1}\circ\varphi_{\mathsf T}$
is a bijection, with inverse $\Gamma_{\mathsf U,\mathsf T}$. For any $\boldsymbol\alpha\in(0,1)^{N-1}$, let 
\[
\boldsymbol\beta:=
\Gamma_{\mathsf T,\mathsf U}(\boldsymbol\alpha).
\]
By the definition of the inverse map,
\(
\varphi_{\mathsf U}(\boldsymbol\beta)=
\varphi_{\mathsf U}\bigl(
\varphi_{\mathsf U}^{-1}
(\varphi_{\mathsf T}(\boldsymbol\alpha))
\bigr)=\varphi_{\mathsf T}(\boldsymbol\alpha).
\)
Thus, the two parameter vectors induce identical global leaf weights. Under linear pooling, the corresponding protocol outputs therefore satisfy
\[
\begin{aligned}
\hat y^{\mathsf U}(\boldsymbol\beta)=\sum_{j=1}^{N}
\varphi_{\mathsf U,j}(\boldsymbol\beta)
\hat y^{(j)}=\sum_{j=1}^{N}
\varphi_{\mathsf T,j}(\boldsymbol\alpha)
\hat y^{(j)}=\hat y^{\mathsf T}(\boldsymbol\alpha).
\end{aligned}
\]
Since $\Gamma_{\mathsf T,\mathsf U}$ is bijective, the two trees have identical attainable output sets. Finally, identical outputs have identical empirical losses; because the benchmark is independent of the
protocol tree, their complementarity values coincide.
\end{proof}

\subsection{Tamari Coordinate Changes and Path Independence}
\label{subsection:tamari}
Corollary~\ref{cor:universal_protocol_equivalence} constructs unique output-preserving local weights for any pair of protocol trees via inverse barycentric maps. We now derive explicit formulas for these transformations under elementary Tamari rotations. The Tamari order is generated by right rotations $((AB)C)\mapsto(A(BC))$, preserving the ordered leaves \citep{tamari1962problemes}. We write $\mathsf T\lessdot_{\mathrm{Tam}}\mathsf T'$ when $\mathsf T'$ covers $\mathsf T$ through one such rotation.

\begin{theorem}[Output-preserving Tamari reparameterization]
\label{thm:tamari_rotation_invariance}
Let \(\mathsf T,\mathsf T'\in\Tcal_N\) satisfy \(\mathsf T\lessdot_{\mathrm{Tam}}\mathsf T'\). Suppose the fragment \(((AB)C)\) in \(\mathsf T\) uses weight \(\gamma\) to combine \(A,B\) and weight \(\delta\) to combine their output with \(C\). After rotating it to \(A(BC)\) in \(\mathsf T'\), leave all other parameters unchanged and set
\begin{equation}
\beta_{A(BC)}=\gamma\delta,
\qquad
\beta_{BC}
=
\frac{\delta(1-\gamma)}{1-\gamma\delta}.
\label{eq:tamari_reparam}
\end{equation}
The resulting Tamari reparameterization map
\(\Gamma_{\mathsf T,\mathsf T'}:(0,1)^{N-1}\to(0,1)^{N-1}\), \(\boldsymbol{\alpha}_{\mathsf T}\mapsto \Gamma_{\mathsf T,\mathsf T'}(\boldsymbol{\alpha}_{\mathsf T}):=\boldsymbol{\alpha}_{\mathsf T'}
\) satisfies
\[
\hat y^{\mathsf T}(\boldsymbol\alpha)
=
\hat y^{\mathsf T'}
\bigl(
\Gamma_{\mathsf T,\mathsf T'}
(\boldsymbol\alpha)
\bigr)
\]
for every choice of leaf predictions. Hence, \(\mathsf T\) and \(\mathsf T'\) have the same complementarity value under any pointwise loss.
\end{theorem}

\begin{proof}
The construction is local. Since $0<\gamma\delta<1$ and \(0<\delta(1-\gamma)<1-\gamma\delta,\) both transported parameters belong to $(0,1)$. Let $z_A,z_B,z_C$ denote the outputs of the subtrees $A,B,C$. Before rotation, the fragment output is
\[
\delta\bigl(\gamma z_A+(1-\gamma)z_B\bigr)
+(1-\delta)z_C.
\]
After rotation, it is
\[
\beta_{A(BC)}z_A+
(1-\beta_{A(BC)})
\bigl(\beta_{BC}z_B+(1-\beta_{BC})z_C\bigr).
\]
Substituting the transported parameters gives
the same coefficients
$\gamma\delta$, $\delta(1-\gamma)$, and
$1-\delta$ on $z_A,z_B,z_C$, respectively.
Since all other local parameters are unchanged,
the full root outputs coincide.
\end{proof}

For a directed Tamari path
\(p=(\mathsf T_0\lessdot_{\mathrm{Tam}}\cdots
\lessdot_{\mathrm{Tam}}\mathsf T_r)\), define
\(
\Gamma_p
:=
\Gamma_{\mathsf T_{r-1},\mathsf T_r}
\circ\cdots\circ
\Gamma_{\mathsf T_0,\mathsf T_1}.
\) The following result establishes the \emph{global coherence of Tamari reparameterization}: between fixed endpoint trees, the induced Tamari reparameterization is independent of the chosen directed Tamari path.

\begin{corollary}[Path independence of Tamari reparameterizations]
\label{thm:tamari_path_independence}
Let $p$ and $q$ be any two directed Tamari paths from $\mathsf T$ to $\mathsf U$. Then
\[
\Gamma_p=\Gamma_q
=\varphi_{\mathsf U}^{-1}\circ\varphi_{\mathsf T}.
\]
\end{corollary}

\begin{proof}
Composition of the coordinate transitions along either path cancels all intermediate barycentric maps, leaving $\varphi_{\mathsf U}^{-1}\circ\varphi_{\mathsf T}$.
\end{proof}

\begin{corollary}[The \(N=4\) pentagon identity]
\label{cor:stasheff_pentagon}
The two directed Tamari paths from the left comb
\(\mathsf T_1=(((12)3)4)\) to the right comb
\(\mathsf T_5=(1(2(34)))\) induce the same reparameterization.
A constructive verification, including the explicit rational compositions
and the corresponding Stasheff pentagon
\citep{Stasheff1963HSpacesI}, is provided in Appendix~\ref{stasheff_pentagon}. We display it in Figure~\ref{fig:stasheff_pentagon_N4}.
\end{corollary}

Universal equivalence holds between any pair of trees with the same ordered leaves, independently of their comparability in the Tamari order. Elementary Tamari rotations provide explicit local realizations of the corresponding coordinate transformations, while path independence guarantees their coherence.

\subsection{Admissibility Constraints and Protocol Equivalence}
\label{subsection:constraints}
Universal protocol equivalence assumes unrestricted local aggregation weights. In real-world HAIs, however, professional hierarchies, organizational roles, and epistemic requirements may restrict admissible weight assignments, potentially depending on protocol structure. These constraints may prevent the universal protocol equivalence, as we show in what follows.

Let $\mathcal A_{\mathsf T}\subseteq(0,1)^{N-1}$ denote the admissible local-weight set for a protocol tree $\mathsf T$, and define its corresponding \emph{feasible global-weight set} by
\[
\Omega_{\mathsf T}
:=
\varphi_{\mathsf T}(\mathcal A_{\mathsf T})
\subseteq\operatorname{int}(\Delta^{N-1}).
\]
By the definition of $\Gamma_{\mathsf T,\mathsf U}$,
\(
\Gamma_{\mathsf T,\mathsf U}(\mathcal A_{\mathsf T})
\subseteq\mathcal A_{\mathsf U}
\Longleftrightarrow
\Omega_{\mathsf T}\subseteq\Omega_{\mathsf U}.
\)
Thus, mutual admissibility preservation is equivalent to equality of the feasible global-weight sets. We consider two families of interaction-level constraints inspired by real-world HAIs: hierarchical priority and epistemic confidence.

\medskip
\noindent\textbf{Hierarchical priority.}
Professional hierarchies may reflect organizational authority and differences in expertise, e.g, in medicine with resident physicians and junior doctors. Analogously, machine-only workflows may assign coordinating and specialized roles to different AI systems, such as in AI agentic workflows. Consider the ordered configuration $(\texttt{expert},\texttt{assistant},\texttt{AI})$ and the two protocol trees $\mathsf T_L=((12)3)$ and $\mathsf T_R=(1(23))$ from Section~\ref{subsection:N_3_eq_locus}. We impose the following hierarchical constraint: at every internal node along the path from the \texttt{expert} leaf to the root, the branch containing the \texttt{expert} must receive more weight than its sibling. Under the internal-node ordering, the admissible sets are
\begin{align*}
\mathcal A_{\mathsf T_L}
&=\{(\alpha_1,\alpha_2)\in(0,1)^2:
\alpha_1>\tfrac12,\ \alpha_2>\tfrac12\},\\
\mathcal A_{\mathsf T_R}
&=\{(\alpha_1,\alpha_2)\in(0,1)^2:
\alpha_1>\tfrac12\}.
\end{align*}
The Tamari reparameterization does not generally preserve these constraints. Indeed, \((0.7,0.7)\in\mathcal A_{\mathsf T_L}\), but
\(\bigl[\Gamma_{\mathsf T_L,\mathsf T_R}(0.7,0.7)\bigr]_1=0.7^2<\tfrac12.\)
In terms of feasible global weights, these constraints yield
$\Omega_{\mathsf T_R}\subsetneq\Omega_{\mathsf T_L}$. In fact,  using the recursive construction of the barycentric maps and Proposition~\ref{prop:barycentric}, we obtain
\(\varphi_{\mathsf T_L}(\alpha_1,\alpha_2)
=(\alpha_1\alpha_2,(1-\alpha_1)\alpha_2,1-\alpha_2)\)
and
\(\varphi_{\mathsf T_R}(\alpha_1,\alpha_2)
=(\alpha_1,(1-\alpha_1)\alpha_2,
(1-\alpha_1)(1-\alpha_2))\). Consequently, the feasible global-weight sets are
\(\Omega_{\mathsf T_L}
=\{\omega\in\operatorname{int}(\Delta^2):
\omega_1>\omega_2,\ \omega_1+\omega_2>\omega_3\}\)
and
\(\Omega_{\mathsf T_R}
=\{\omega\in\operatorname{int}(\Delta^2):
\omega_1>\omega_2+\omega_3\}\).
Hence, \(\Omega_{\mathsf T_R}\subsetneq
\Omega_{\mathsf T_L}\). In summary, admissibility is preserved from $\mathsf T_R$ to $\mathsf T_L$, but not conversely.

\medskip
\noindent\textbf{Epistemic confidence.}
A second family of constraints concerns confidence in individual predictions.  We represent confidence through agent-specific maps
\[
C_j:\Yhat^n\to[0,1]^n,\qquad
C_j(\hat y^{(j)})=(c_{j,1},\dots,c_{j,n}),
\]
where $c_{j,i}$ denotes the confidence of agent $j$ in its prediction on case $i$. These maps may differ across agents, and comparisons require epistemically justified, comparable confidence scores. We impose the following constraint: whenever two individual agents are directly aggregated, the agent with strictly higher confidence on every case must receive more than half of the local weight. No condition is imposed when confidence rankings differ across cases or when an input is an intermediate aggregate, for which confidence has not been defined.

Consider again $(\texttt{expert},\texttt{assistant},\texttt{AI})$ with $\mathsf T_L=((12)3)$ and $\mathsf T_R=(1(23))$. Suppose that, on every case, the confidence ranking is \(c_{1,i}>c_{3,i}>c_{2,i},\ i=1,\dots,n.\)
The corresponding admissible sets are
\begin{align*}
\mathcal A_{\mathsf T_L}
&=\{(\alpha_1,\alpha_2)\in(0,1)^2:
\alpha_1>\tfrac12\},\\
\mathcal A_{\mathsf T_R}
&=\{(\alpha_1,\alpha_2)\in(0,1)^2:
\alpha_2<\tfrac12\}.
\end{align*}
Indeed, $\mathsf T_L$ directly aggregates the \texttt{expert} and \texttt{assistant}, while $\mathsf T_R$ directly aggregates the \texttt{assistant} and \texttt{AI}. The Tamari reparameterization does not generally preserve these constraints. For instance,
\((0.6,0.8)\in\mathcal A_{\mathsf T_L}\), but
\(\bigl[\Gamma_{\mathsf T_L,\mathsf T_R}
(0.6,0.8)\bigr]_2=\frac{0.8(1-0.6)}{1-0.6\cdot0.8}=\frac{8}{13}>\frac12.\)
Conversely, $(0.2,0.4)\in\mathcal A_{\mathsf T_R}$, but its inverse transport has first component $5/13<1/2$, violating the left-tree constraint.
The corresponding feasible global-weight sets $\Omega_{\mathsf T_L}$ and
$\Omega_{\mathsf T_R}$ are incomparable under inclusion. Hence, admissibility preservation fails in both directions.

\medskip
Hierarchical and epistemic requirements may apply simultaneously.  Their intersection may reinforce existing requirements or introduce additional restrictions; it need not be nonempty when hierarchical and epistemic priorities
conflict. The feasible global-weight-set criterion applies equally to these combined constraints.

\begin{figure*}[h!]
\centering
\begin{tikzpicture}[
    scale=1,
    >=Latex,
    associahedron/.style={line width=0.9pt, black!55},
    upperpath/.style={->, very thick, black, shorten <=15.5mm, shorten >=15.5mm},
    lowerpath/.style={->, very thick, black, shorten <=15.5mm, shorten >=15.5mm},
    edgelabel/.style={font=\scriptsize, fill=white, inner sep=1pt},
    treeedge/.style={line width=0.5pt, line cap=round},
    inode/.style={circle, fill=black, inner sep=0pt},
    leaf/.style={circle, fill=black, inner sep=0pt},
    leaflabel/.style={font=\tiny},
    treelabel/.style={font=\small},
    vertexcircle/.style={draw=black, fill=white, line width=1.2pt}
]

\def\Rp{5}      
\def\Rc{1.55}     
\def\ts{0.55}     

\coordinate (T1) at ( 90:\Rp);
\coordinate (T2) at ( 18:\Rp);
\coordinate (T4) at (-54:\Rp);
\coordinate (T5) at (-126:\Rp);
\coordinate (T3) at (162:\Rp);

\draw[associahedron] (T1) -- (T2) -- (T4) -- (T5) -- (T3) -- cycle;

\draw[lowerpath] (T1) -- node[edgelabel, above right] {$\Gamma_{1,2}$} (T2);
\draw[lowerpath] (T2) -- node[edgelabel, below right] {$\Gamma_{2,4}$} (T4);
\draw[lowerpath] (T4) -- node[edgelabel, below] {$\Gamma_{4,5}$} (T5);

\draw[upperpath] (T1) -- node[edgelabel, above left] {$\Gamma_{1,3}$} (T3);
\draw[upperpath] (T3) -- node[edgelabel, below left] {$\Gamma_{3,5}$} (T5);

\node[align=center, font=\small] at (0,0.45)
{
$\Gamma_{3,5}\circ\Gamma_{1,3}
=
\Gamma_{4,5}\circ\Gamma_{2,4}\circ\Gamma_{1,2}$
};

\node[align=center, font=\scriptsize] at (0,-0.55)
{
$\displaystyle
(\alpha_1,\alpha_2,\alpha_3)
\longmapsto
\left(
\alpha_1\alpha_2\alpha_3,
\frac{(1-\alpha_1)\alpha_2\alpha_3}{1-\alpha_1\alpha_2\alpha_3},
\frac{(1-\alpha_2)\alpha_3}{1-\alpha_2\alpha_3}
\right)
$
};

\begin{scope}[shift={(T1)}]
    \draw[vertexcircle] (0,0) circle (\Rc);

    \begin{scope}[scale=\ts, transform shape]

        \node[leaf] at (-1.8,1.25) {};
        \node[leaf] at (-0.6,1.25) {};
        \node[leaf] at (0.6,1.25) {};
        \node[leaf] at (1.8,1.25) {};

        \node[inode] at (-1.2,0.55) {};
        \node[inode] at (-0.6,-0.15) {};
        \node[inode] at (0,-0.85) {};

        \draw[treeedge] (-1.8,1.25) -- (-1.2,0.55);
        \draw[treeedge] (-0.6,1.25) -- (-1.2,0.55);
        \draw[treeedge] (-1.2,0.55) -- (-0.6,-0.15);
        \draw[treeedge] (0.6,1.25) -- (-0.6,-0.15);
        \draw[treeedge] (-0.6,-0.15) -- (0,-0.85);
        \draw[treeedge] (1.8,1.25) -- (0,-0.85);

        \node[treelabel] at (0,-1.55) {$\mathsf T_1$};
    \end{scope}
\end{scope}

\begin{scope}[shift={(T2)}]
    \draw[vertexcircle] (0,0) circle (\Rc);

    \begin{scope}[scale=\ts, transform shape]

        \node[leaf] at (-1.8,1.25) {};
        \node[leaf] at (-0.6,1.25) {};
        \node[leaf] at (0.6,1.25) {};
        \node[leaf] at (1.8,1.25) {};

        \node[inode] at (0,0.55) {};
        \node[inode] at (-0.6,-0.15) {};
        \node[inode] at (0,-0.85) {};

        \draw[treeedge] (-0.6,1.25) -- (0,0.55);
        \draw[treeedge] (0.6,1.25) -- (0,0.55);
        \draw[treeedge] (-1.8,1.25) -- (-0.6,-0.15);
        \draw[treeedge] (0,0.55) -- (-0.6,-0.15);
        \draw[treeedge] (-0.6,-0.15) -- (0,-0.85);
        \draw[treeedge] (1.8,1.25) -- (0,-0.85);

        \node[treelabel] at (0,-1.55) {$\mathsf T_2$};
    \end{scope}
\end{scope}

\begin{scope}[shift={(T4)}]
    \draw[vertexcircle] (0,0) circle (\Rc);

    \begin{scope}[scale=\ts, transform shape]

        \node[leaf] at (-1.8,1.25) {};
        \node[leaf] at (-0.6,1.25) {};
        \node[leaf] at (0.6,1.25) {};
        \node[leaf] at (1.8,1.25) {};

        \node[inode] at (0,0.55) {};
        \node[inode] at (0.6,-0.15) {};
        \node[inode] at (0,-0.85) {};

        \draw[treeedge] (-0.6,1.25) -- (0,0.55);
        \draw[treeedge] (0.6,1.25) -- (0,0.55);
        \draw[treeedge] (0,0.55) -- (0.6,-0.15);
        \draw[treeedge] (1.8,1.25) -- (0.6,-0.15);
        \draw[treeedge] (-1.8,1.25) -- (0,-0.85);
        \draw[treeedge] (0.6,-0.15) -- (0,-0.85);

        \node[treelabel] at (0,-1.55) {$\mathsf T_4$};
    \end{scope}
\end{scope}

\begin{scope}[shift={(T5)}]
    \draw[vertexcircle] (0,0) circle (\Rc);

    \begin{scope}[scale=\ts, transform shape]

        \node[leaf] at (-1.8,1.25) {};
        \node[leaf] at (-0.6,1.25) {};
        \node[leaf] at (0.6,1.25) {};
        \node[leaf] at (1.8,1.25) {};

        \node[inode] at (1.2,0.55) {};
        \node[inode] at (0.6,-0.15) {};
        \node[inode] at (0,-0.85) {};

        \draw[treeedge] (0.6,1.25) -- (1.2,0.55);
        \draw[treeedge] (1.8,1.25) -- (1.2,0.55);
        \draw[treeedge] (-0.6,1.25) -- (0.6,-0.15);
        \draw[treeedge] (1.2,0.55) -- (0.6,-0.15);
        \draw[treeedge] (-1.8,1.25) -- (0,-0.85);
        \draw[treeedge] (0.6,-0.15) -- (0,-0.85);

        \node[treelabel] at (0,-1.55) {$\mathsf T_5$};
    \end{scope}
\end{scope}

\begin{scope}[shift={(T3)}]
    \draw[vertexcircle] (0,0) circle (\Rc);

    \begin{scope}[scale=\ts, transform shape]

        \node[leaf] at (-1.8,1.25) {};
        \node[leaf] at (-0.6,1.25) {};
        \node[leaf] at (0.6,1.25) {};
        \node[leaf] at (1.8,1.25) {};

        \node[inode] at (-1.2,0.55) {};
        \node[inode] at (1.2,0.55) {};
        \node[inode] at (0,-0.85) {};

        \draw[treeedge] (-1.8,1.25) -- (-1.2,0.55);
        \draw[treeedge] (-0.6,1.25) -- (-1.2,0.55);
        \draw[treeedge] (0.6,1.25) -- (1.2,0.55);
        \draw[treeedge] (1.8,1.25) -- (1.2,0.55);
        \draw[treeedge] (-1.2,0.55) -- (0,-0.85);
        \draw[treeedge] (1.2,0.55) -- (0,-0.85);

        \node[treelabel] at (0,-1.55) {$\mathsf T_3$};
    \end{scope}
\end{scope}

\end{tikzpicture}

\caption{The pentagon identity satisfied by the Tamari-cover reparameterizations for \(N\!=\!4\). The vertices of the \(1\)-skeleton of the associahedron \(\mathsf{Assoc}_4\)  are the five rooted planar binary trees in \(\Tcal_4\). Edges represent Tamari cover moves. The two highlighted directed paths from the left comb \(\mathsf T_1\) to the right comb \(\mathsf T_5\) induce the same reparameterization. By Theorem~\ref{thm:tamari_rotation_invariance}, both paths preserve the root output and therefore the tree-relative complementarity functional.}
\label{fig:stasheff_pentagon_N4}
\end{figure*}

\section{Complementarity in Regression Under Squared Loss}
\label{section:regression}
We now turn to regression. The label and prediction spaces are $\Y=\Yhat=\R$, so prediction vectors lie in \(\Yhat^n=\R^n\) and the loss is the squared loss, i.e., $\ell(y,\hat{y})=(y-\hat{y})^2$.

\begin{proposition}[Squared-loss complementarity as distance minimization]
\label{prop:squared_loss_geometry}
Assume regression with squared loss. Let 
\[
K_{n,\mathrm{ba}}
:=
\min_{1\le j\le N}
\frac1n\sum_{i=1}^n
(y_i-\hat y_i^{(j)})^2,
\qquad
K_{n,\mathrm{pw}}
:=
\frac1n\sum_{i=1}^n
\min_{1\le j\le N}
(y_i-\hat y_i^{(j)})^2.
\]
Then, for every protocol tree \(\mathsf T\in \Tcal_N\),
\begin{equation}
n\Psi_{\mathsf T,\diamondsuit}^{m_{\mathsf T}}
=
nK_{n,\diamondsuit}
-
\|y-\hat y_{\mathsf T}\|_2^2,
\qquad
\diamondsuit\in\{\mathrm{ba},\mathrm{pw}\}.
\label{eq:squared_loss_geometry}
\end{equation}
Consequently:
\begin{enumerate}[leftmargin=*]
    \item Maximizing complementarity is equivalent to minimizing the Euclidean distance between \(\hat y^{\mathsf T}\) and \(y\);
    \item \(\Psi_{\mathsf T,\diamondsuit}^{m_{\mathsf T}}>0\) exactly when
    \(\hat y^{\mathsf T}\) lies in the open ball centered at \(y\) with radius
    \(\sqrt{nK_{n,\diamondsuit}}\), and equality holds on its boundary;
    \item Two protocols have equal complementarity exactly when their outputs  are equidistant from \(y\); thus, equal complementarity does not require equal protocol outputs.
\end{enumerate}
\end{proposition}
\begin{proof}
Considering~\eqref{eq:compl_functional} and substituting the squared loss into \(n\Psi_{\mathsf T,\diamondsuit}^{m_{\mathsf T}}\) yields the constant \(nK_{n,\diamondsuit}\) minus the squared distance of the protocol tree output \(\hat y^{\mathsf T}\) from \(y\), i.e.,~\eqref{eq:squared_loss_geometry}.
\end{proof}

Further, in regression under the squared loss function, \emph{complementarity invariance under HAI protocol change} has a direct geometric interpretation. By Proposition~\ref{prop:squared_loss_geometry}, two distinct trees \(\mathsf T,\mathsf T'\in\Tcal_N\) have the same complementarity value 
\begin{align}
&\Psi_{\mathsf{T},\diamondsuit}^{m_{\mathsf{T}}}
(\hat y^{(1)},\dots,\hat y^{(N)};D)
=
\Psi_{\mathsf{T}',\diamondsuit}^{m_{\mathsf{T}'}}
(\hat y^{(1)},\dots,\hat y^{(N)};D)
\Longleftrightarrow 
\|y-\hat y_{\mathsf{T}}\|_2^2=
\|y-\hat y_{\mathsf{T}'}\|_2^2 .
\label{eq:compl_inv_geom}
\end{align}
Geometrically, \eqref{eq:compl_inv_geom} holds when the two protocol outputs \(\hat y_{\mathsf T}\) and \(\hat y_{\mathsf T'}\) lie on the same sphere centered at \(y\). Thus, equality of complementarity values does not require equality of protocol outputs: two distinct protocols of the same HAI may achieve the same level of complementarity while producing different predictions. Under linear pooling, the parameter values for which this protocol-indifference condition holds form an algebraic locus:

\begin{proposition}[Protocol-indifference locus under linear pooling]
\label{prop:equality_locus}
Assume regression under squared loss and fix leaf predictions
\(\hat y^{(1)},\dots,\hat y^{(N)}\in\mathbb R^n\). Let
\(\mathsf T,\mathsf T'\in\Tcal_N\). For each tree, fix the canonical recursive
left-to-right ordering of its \(N-1\) internal nodes and assign the local
linear-pooling rule
\(
m^{\mathrm{id}}_{2,\alpha_k}(u,v)=\alpha_k u+(1-\alpha_k)v
\)
to the \(k\)-th internal node. For
\(\boldsymbol\alpha=(\alpha_1,\dots,\alpha_{N-1})\in[0,1]^{N-1}\), write
\(\hat y_{\mathsf T}(\boldsymbol\alpha)\) and
\(\hat y_{\mathsf T'}(\boldsymbol\alpha)\) for the corresponding root outputs. Then the set of local weights for which the two protocol trees have the same
complementarity value is
\[
\mathcal S_{\mathsf T,\mathsf T'}
:=
\left\{
\boldsymbol\alpha\in[0,1]^{N-1}:
P_{\mathsf T,\mathsf T'}(\boldsymbol\alpha)=0
\right\}, 
P_{\mathsf T,\mathsf T'}(\boldsymbol\alpha)
:=
\|y-\hat y_{\mathsf T}(\boldsymbol\alpha)\|_2^2
-
\|y-\hat y_{\mathsf T'}(\boldsymbol\alpha)\|_2^2.
\]
\(P_{\mathsf T,\mathsf T'}\) is a polynomial in
\(\boldsymbol\alpha\) of degree at most \(2(N-1)\). $\mathcal S_{\mathsf T,\mathsf T'}$ is nonempty: it always contains the two corners
\(
(0,\dots,0)\) and \((1,\dots,1)\).
\end{proposition}

\begin{proof}
The benchmark term \(\Phi_\diamondsuit\) is independent of the protocol tree. Hence two
trees have the same complementarity value if and only if their squared-loss
interaction terms are equal, equivalently
\(P_{\mathsf T,\mathsf T'}(\boldsymbol\alpha)=0\). For linear local
composition, each coordinate of \(\hat y_{\mathsf T}(\boldsymbol\alpha)\) and
\(\hat y_{\mathsf T'}(\boldsymbol\alpha)\) is a polynomial in
\(\boldsymbol\alpha\) of degree at most \(N-1\), since each leaf contribution is
weighted by a product of local weights along a root-to-leaf path. Therefore
each squared norm has degree at most \(2(N-1)\), and so does
\(P_{\mathsf T,\mathsf T'}\). Finally, if
\(\boldsymbol\alpha=(1,\dots,1)\), every internal node selects its left input, so every tree outputs the leftmost leaf prediction \(\hat y^{(1)}\). If \(\boldsymbol\alpha=(0,\dots,0)\), every internal node selects its right input, so every tree outputs the rightmost leaf prediction \(\hat y^{(N)}\). 
\end{proof}
We now study complementarity in regression under squared loss in low-\(N\)
cases.

\subsection{The $N\!=\!2$ Case: Complementarity Maximization} 
Choose the agent configuration $(\textnormal{\texttt{human}}, \textnormal{\texttt{AI}})$ to decorate the leaves of the unique tree $\mathsf{T}\in\Tcal_2$. The case $(\textnormal{\texttt{AI}}, \textnormal{\texttt{human}})$ is treated analogously. 

\begin{proposition}[Maximizing complementarity for \(N\!=\!2\) regression under \\ squared loss and linear pooling]
\label{prop:weighted_regression}
Let \(\mathsf T\) be the unique tree in \(\Tcal_2\). Choose the squared loss and let
\(
m_{2,\alpha}^{\mathrm{id}}(u,v)=\alpha u+(1-\alpha)v\),
\(\alpha\in[0,1]\). Then
\begin{equation}
n\Psi_{\mathsf T,\diamondsuit}^{m_{2,\alpha}^{\mathrm{id}}}
(\hat y^H,\hat y^{AI};D)
=
-A_n\alpha^2-2B_n\alpha+(nK_{n,\diamondsuit}-C_n), \label{eq:regression_quadratic}
\end{equation}
where 
\[
A_n=\|\hat y^H-\hat y^{AI}\|_2^2, \quad B_n=\langle \hat y^H-\hat y^{AI},\hat y^{AI}-y\rangle, \quad  C_n=\|\hat y^{AI}-y\|_2^2.
\]
If \(A_n>0\), the maximizing weight is
\(
\alpha^\ast=
\Pi_{[0,1]}\!\left(-\frac{B_n}{A_n}\right).
\)
If \(A_n=0\), the aggregate is constant in \(\alpha\) and yields no complementarity.
\end{proposition}
\begin{proof}
Write
\[
\sum_{i=1}^n(\alpha \hat y_i^H+(1-\alpha)\hat y_i^{AI}-y_i)^2
=
A_n\alpha^2+2B_n\alpha+C_n,
\]
where
\[
A_n=\|\hat y^H-\hat y^{AI}\|_2^2,\quad
B_n=\langle \hat y^H-\hat y^{AI},\hat y^{AI}-y\rangle,\quad
C_n=\|\hat y^{AI}-y\|_2^2.
\]
Thus 
\[
n\Psi_{\mathsf T,\diamondsuit}^{m_{2,\alpha}^{\mathrm{id}}}
=
-A_n\alpha^2-2B_n\alpha+(nK_{n,\diamondsuit}-C_n),
\]
which is Equation~\eqref{eq:regression_quadratic}, where 
\begin{align*}
&K_{n,\mathrm{ba}}=
\min\left\{
\frac1n\|\hat y^H-y\|_2^2,\,
\frac1n\|\hat y^{AI}-y\|_2^2
\right\},\\
&K_{n,\mathrm{pw}}
=
\frac1n\sum_{i=1}^n
\min\!\left\{(y_i-\hat y_i^H)^2,\,(y_i-\hat y_i^{AI})^2\right\}.
\end{align*}
If \(A_n=0\), then \(\hat y^H=\hat y^{AI}\), hence \(B_n=0\), \(nK_{n,\diamondsuit}=C_n\), and $n\Psi_{\mathsf T,\diamondsuit}^{m_{2,\alpha}^{\mathrm{id}}}=0$ for all $\alpha\in[0,1]$. Thus every \(\alpha\in[0,1]\) is optimal, but no complementarity is possible. 

Assume now that \(A_n>0\). The function
\[
Q(\alpha):=
-A_n\alpha^2-2B_n\alpha+(nK_{n,\diamondsuit}-C_n)
\]
is a strictly concave quadratic on \([0,1]\). Its derivative is $Q'(\alpha)=-2A_n\alpha-2B_n$,
so the unconstrained maximizer is $\alpha_0=-\frac{B_n}{A_n}$. Restricting to \([0,1]\) gives the unique constrained maximizer $\alpha^\ast=\Pi_{[0,1]}\!\left(-\frac{B_n}{A_n}\right)$. Equivalently,
\[
\alpha^\ast
=
\begin{cases}
0, & B_n\ge 0,\\[4pt]
-\dfrac{B_n}{A_n}, & -A_n<B_n<0,\\[10pt]
1, & B_n\le -A_n.
\end{cases}
\]

The optimized value is
\[
n\Psi_{\mathsf T,\diamondsuit}^{m_{2,\alpha^\ast}^{\mathrm{id}}}
=
nK_{n,\diamondsuit}-
\bigl(
A_n(\alpha^\ast)^2+2B_n\alpha^\ast+C_n
\bigr).
\]
In the interior case \(-A_n<B_n<0\), this becomes
\[
n\Psi_{\mathsf T,\diamondsuit}^{m_{2,\alpha^\ast}^{\mathrm{id}}}
=
nK_{n,\diamondsuit}-C_n+\frac{B_n^2}{A_n}.
\]
Equivalently, the optimized squared loss of the aggregate is
\[
C_n-\frac{B_n^2}{A_n}.
\]

At the endpoints, the aggregate reduces to one of the standalone predictions:
\[
\alpha^\ast=0
\quad\Rightarrow\quad
n\Psi_{\mathsf T,\diamondsuit}^{m_{2,0}^{\mathrm{id}}}
=
nK_{n,\diamondsuit}-\|\hat y^{AI}-y\|_2^2\le 0,
\]
and
\[
\alpha^\ast=1
\quad\Rightarrow\quad
n\Psi_{\mathsf T,\diamondsuit}^{m_{2,1}^{\mathrm{id}}}
=
nK_{n,\diamondsuit}-\|\hat y^H-y\|_2^2\le 0.
\]
Therefore complementarity, when it occurs, can only occur in the interior case.
\end{proof}
Proposition~\ref{prop:weighted_regression} admits an explicit geometric interpretation. Let us rewrite $B_n$ as $B_n=\|\hat y^H-\hat y^{AI}\|_2\|\hat y^{AI}-y\|_2\cos\theta$. Then, the interior case $-\frac{B_n}{A_n}\in(0,1)$ is equivalent to the two inequalities
\begin{equation}
\cos\theta<0
\qquad\text{and}\qquad
\frac{\|\hat y^H-\hat y^{AI}\|_2}
{\|\hat y^{AI}-y\|_2}
>
-\cos\theta. \label{eq:interior_case_ineq}
\end{equation}
The first condition says that the human--AI disagreement direction $\hat y^H-\hat y^{AI}$ has a component pointing \emph{against} the AI residual $\hat y^{AI}-y$. The second condition says that the human prediction $\hat y^{H}$ moves sufficiently far in that corrective direction for the projection of \(y\) onto the human--AI line $\hat y^H-\hat y^{AI}$ to lie inside the convex segment between \(\hat y^{AI}\) and \(\hat y^H\). 

\begin{corollary}
In the interior case
$-\frac{B_n}{A_n}\in(0,1)$, the optimized value of the complementarity functional is
\begin{align}
n\Psi_{\mathsf T,\diamondsuit}^{m_{2,\alpha^\ast}^{\mathrm{id}}}
(\hat y^H,\hat y^{AI};D)=nK_{n,\diamondsuit}-C_n+\frac{B_n^2}{A_n}=nK_{n,\diamondsuit}-\|\hat y^{AI}-y\|_2^2\sin^2\theta,\label{eq:interior_sin}
\end{align}
where $\alpha^{\ast}\in(0,1)$. Thus, linear pooling can remove only the component of the AI residual that lies along the human--AI disagreement direction. In particular, conditional on a fixed benchmark value \(K_{n,\diamondsuit}\) and a fixed AI residual norm \(\|\hat y^{AI}-y\|_2\), the optimized complementarity value is maximized by minimizing \(\sin^2\theta\), which occurs when the human--AI disagreement direction is opposite to the AI residual as in Figure~\ref{fig:projection_geometry_n2}.
\end{corollary}

\begin{figure}[t]
\centering
\begin{tikzpicture}[
    scale=0.6,
    every node/.style={font=\small},
    point/.style={circle, fill=black, inner sep=1.6pt},
    projpoint/.style={circle, draw=black, fill=white, inner sep=1.4pt},
    vec/.style={->, very thick, >=Stealth},
    line/.style={thick},
    dashedline/.style={dashed, thick},
    diff/.style={->, densely dashed, thick, >=Stealth},
    dasheddiff/.style={->, dashed, thick, >=Stealth},
    underlying/.style={gray, semithick},
    projseg/.style={densely dotted, semithick}
]

\coordinate (O)  at (0,0);
\coordinate (Y)  at (1.55,0.75);
\coordinate (AI) at (2.2,2.5);
\coordinate (H)  at (3.5,-1.55);

\coordinate (R)  at ($(AI)-(Y)$);   
\coordinate (D)  at ($(H)-(AI)$);   
\coordinate (HY) at ($(H)-(Y)$);    
\coordinate (AH) at ($(AI)-(H)$);   

\coordinate (P) at ($(O)!(R)!(D)$);

\draw[underlying] ($(O)!-1.0!(D)$) -- ($(O)!1.2!(D)$);

\draw[underlying] ($(AI)!-0.45!(H)$) -- ($(AI)!1.35!(H)$);

\node[point, label=above right:{$y$}] at (Y) {};
\node[point, label=above right:{$\hat y^{AI}$}] at (AI) {};
\node[point, label=below right:{$\hat y^{H}$}] at (H) {};
\node[point, label=left:{$\mathbf{0}$}] at (O) {};

\draw[vec] (O) -- (AI);
\draw[vec] (O) -- (Y);
\draw[vec] (O) -- (H);

\draw[diff] (O) -- (R)
    node[above] {$\hat y^{AI}-y$};

\draw[diff] (O) -- (D)
    node[left] {$\hat y^{H}-\hat y^{AI}$};


\draw[projseg] (R) -- (P);

\node[anchor=west] (orthlabel) at (-5.5,2.5) {$\|\hat y^{AI}-y\|_2\sin\theta$};

\draw[
    <-,
    decorate,
    decoration={snake, amplitude=1pt, segment length=4.5mm},
    thin
]
(orthlabel.east) -- ($(R)!0.55!(P)$);

\draw[projseg] (O) -- (P);

\node[anchor=west] (projlabel) at (-6.5,1.45)
    {$-\|\hat y^{AI}-y\|_2\cos\theta$};

\draw[
    <-,
    decorate,
    decoration={snake, amplitude=1pt, segment length=4.5mm},
    very thin
]
(projlabel.east) -- ($(O)!0.55!(P)$);

\pic[
    draw,
    angle radius=0.5cm,
    angle eccentricity=1.25,
    "$\theta$"
] {angle = D--O--R};

\end{tikzpicture}

\caption{Geometry for the \(N\!=\!2\) regression case with linear pooling. The segment of the line through \(\hat y^{AI}\) and \(\hat y^{H}\) that lies between the two vectors is the locus of feasible human--AI team predictions. The line through the origin $\mathbf{0}$ is the human--AI disagreement direction. The vector \(\hat y^{AI}-y\) is projected onto this line and \(\theta\) is the angle between \(\hat y^{AI}-y\) and \(\hat y^{H}-\hat y^{AI}\). The inequalities~\eqref{eq:interior_case_ineq} are both satisfied.
}
\label{fig:projection_geometry_n2}
\end{figure}

\subsubsection{Numerical illustration of the \(N\!=\!2\) regression case}\label{num_exp1} We illustrate the \(N\!=\!2\) regression case with linear pooling by training a random-forest regressor on the California housing training set and evaluating complementarity on the  test set. All details on the experimental setting are in Appendix~\ref{ml_details_exp1}. Taking the AI residual \(r\!=\!\hat y^{AI}\!-\!y\) as fixed, we generate synthetic human predictions of the form \(\hat y^H=\hat y^{AI}+d\), where the displacement \(d\) has controlled angle \(\theta\) with \(r\) and controlled norm ratio \(q=\|d\|_2/\|r\|_2\). For each scenario, we plot the quadratic \(n\Psi_{\mathsf T}^{m_{2,\alpha}^{\mathrm{id}}}\) and its constrained maximizer \(\alpha^\ast=\Pi_{[0,1]}(-B_n/A_n)\). Figure~\ref{fig:exp1_summary} confirms the geometric setting: when the human displacement points in the same direction as the AI error, or is orthogonal to it, the optimizer collapses to the AI boundary and no complementarity is obtained. A corrective direction is not sufficient by itself: at \(\theta=\frac{3\pi}{4}\) with small \(q\), the unconstrained maximizer lies beyond the feasible segment, so the constrained optimum is at the human boundary and remains non-complementary. Positive complementarity appears only when the human displacement is both corrective and sufficiently long, as in the \(\theta=\frac{3\pi}{4}\) and \(\theta=\pi\) cases with interior optima.

\begin{figure}[h!]
    \centering
    \includegraphics[width=0.8\textwidth]{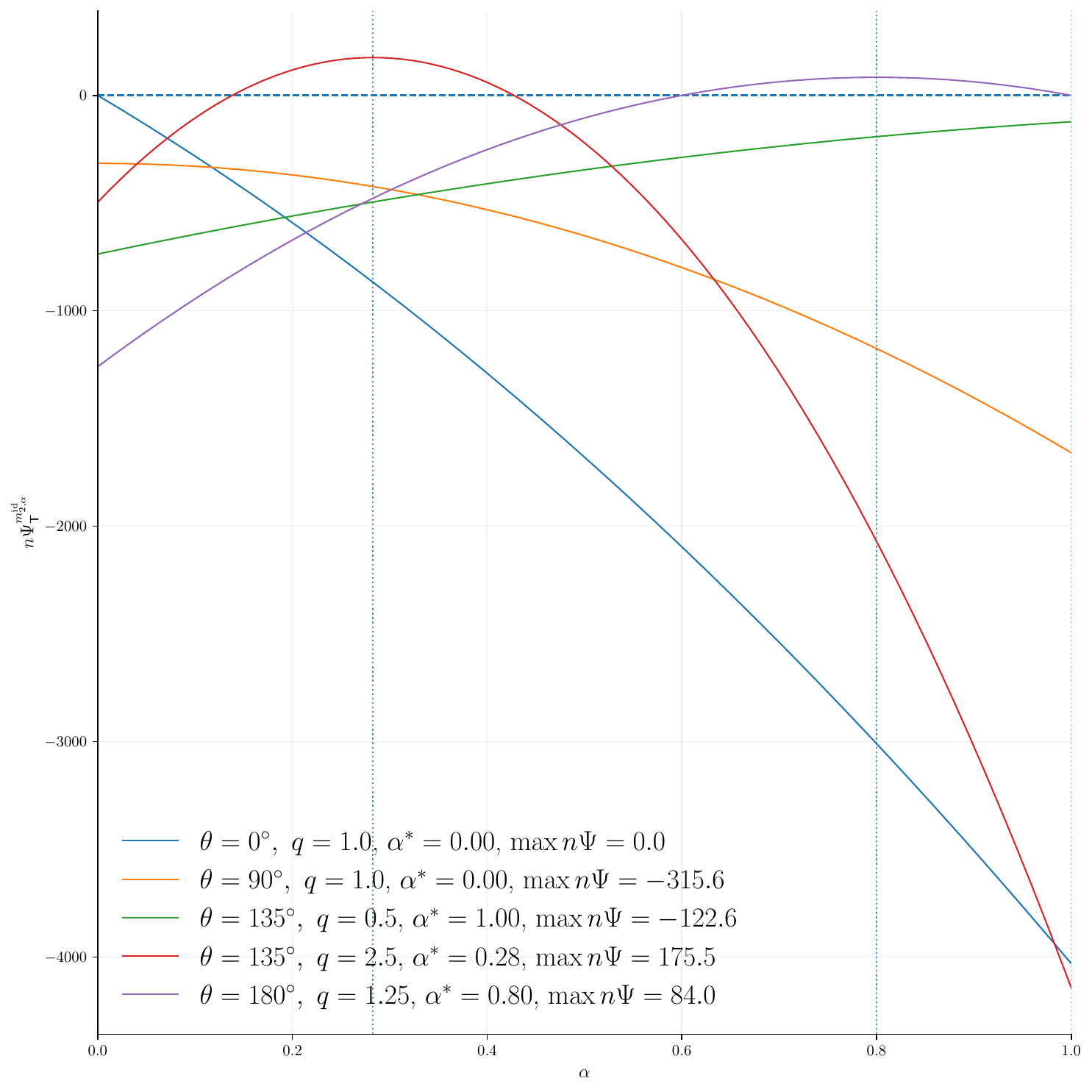}
        \caption{\(N\!=\!2\) regression complementarity under linear pooling on the California housing dataset. Curves show \(n\Psi_{\mathsf T}^{m_{2,\alpha}^{\mathrm{id}}}\) as a function of \(\alpha\) for synthetic human predictions constructed by controlling the angle \(\theta\) between \(\hat y^{H}-\hat y^{AI}\) and \(\hat y^{AI}-y\), and the relative displacement \(q=\|\hat y^{H}-\hat y^{AI}\|_2/\|\hat y^{AI}-y\|_2\). Vertical markers indicate the constrained optimizer. \(nK_{n,\diamondsuit}\) changes as a function of \(\hat y^{H}\).}
    \label{fig:exp1_summary}
\end{figure}

\subsection{The \(N\!=\!3\) Case: Fixed-Parameter Protocol Indifference}
\label{subsection:N_3_eq_locus}
Let \(\mathsf T_L=((12)3)\) and \(\mathsf T_R=(1(23))\) be the two rooted planar binary trees in \(\Tcal_3\), and decorate their leaves with the fixed ordered configuration $(\texttt{expert},\texttt{assistant},\texttt{AI})$ as in Figure~\ref{fig:N3_trees_horizontal}. Thus, \(\mathsf T_L\) represents the protocol in which the \texttt{expert} and \texttt{assistant} first form an intermediate human judgment, which is then combined with the \texttt{AI} prediction. By contrast, \(\mathsf T_R\) represents the protocol in which the \texttt{assistant} first interacts with the \texttt{AI}, and the \texttt{expert} enters only at the final stage. 

We compare the two protocols under linear pooling using the \emph{same} parameter vector $\boldsymbol\alpha=(\alpha_1,\alpha_2)\in[0,1]^2$, assigned according to each tree's internal-node ordering. Unlike the universal equivalence established in Section~\ref{section:linear_protocol_coordinates}, no reparameterization is applied. Since
\begin{equation}
\hat y^{\mathsf T_L}(\boldsymbol\alpha)
-\hat y^{\mathsf T_R}(\boldsymbol\alpha)
=
\alpha_1(1-\alpha_2)
(\hat y^{(3)}-\hat y^{(1)}),
\label{eq:N3_difference}
\end{equation}
the outputs generally differ. We investigate when they nevertheless yield identical complementarity values.

\def\treescale{0.6}

\begin{figure}[h!]
\centering
\begin{tikzpicture}[
    scale=\treescale,
    transform shape,
    every node/.style={font=\small},
    leaf/.style={inner sep=1pt, outer sep=0pt},
    vec/.style={inner sep=1pt, outer sep=0pt},
    inode/.style={circle, fill=black, inner sep=0pt, outer sep=0pt},
    edge/.style={line width=0.5pt, line cap=round},
    predlink/.style={
        ->,
        semithick,
        decorate,
        decoration={snake, amplitude=0.5pt, segment length=2.5mm},
        shorten <=1pt,
        shorten >=1pt
    },
    rootlink/.style={
    ->,
    semithick,
    shorten <=1pt,
    shorten >=1pt
}
]

\begin{scope}[xshift=0cm]

\coordinate (l1) at (0,4);
\coordinate (l2) at (2,4);
\coordinate (l3) at (4,4);
\coordinate (a1) at (1,3);
\coordinate (a2) at (2,2);

\node[vec] (P1) at (0,5) {$\hat y^{(1)}$};
\node[vec] (P2) at (2,5) {$\hat y^{(2)}$};
\node[vec] (P3) at (4,5) {$\hat y^{(3)}$};

\draw[edge] (l1) -- (a1);
\draw[edge] (l2) -- (a1);
\draw[edge] (a1) -- (a2);
\draw[edge] (l3) -- (a2);

\node[leaf, anchor=south] (L1) at (l1) {\texttt{expert}};
\node[leaf, anchor=south] (L2) at (l2) {\texttt{assistant}};
\node[leaf, anchor=south] (L3) at (l3) {\texttt{AI}};

\draw[predlink] (P1.south) -- (L1.north);
\draw[predlink] (P2.south) -- (L2.north);
\draw[predlink] (P3.south) -- (L3.north);

\node[inode, label=left:{$\alpha_1$}]  at (a1) {};

\node[inode, label=right:{$\alpha_2$}] (RootL) at (a2) {};

\node (OutL) at (2,1) {$\hat{y}_{\mathsf{T}_L}$};

\draw[rootlink] (RootL.south) -- (OutL.north);

\node at (1.9,0) {$\mathsf{T}_L=((12)3)$};

\end{scope}

\begin{scope}[xshift=6.8cm]

\coordinate (r1) at (0,4);
\coordinate (r2) at (2,4);
\coordinate (r3) at (4,4);
\coordinate (b1) at (3,3);
\coordinate (b2) at (2,2);

\node[vec] (Q1) at (0,5) {$\hat y^{(1)}$};
\node[vec] (Q2) at (2,5) {$\hat y^{(2)}$};
\node[vec] (Q3) at (4,5) {$\hat y^{(3)}$};

\draw[edge] (r2) -- (b1);
\draw[edge] (r3) -- (b1);
\draw[edge] (r1) -- (b2);
\draw[edge] (b1) -- (b2);

\node[leaf, anchor=south] (R1) at (r1) {\texttt{expert}};
\node[leaf, anchor=south] (R2) at (r2) {\texttt{assistant}};
\node[leaf, anchor=south] (R3) at (r3) {\texttt{AI}};

\draw[predlink] (Q1.south) -- (R1.north);
\draw[predlink] (Q2.south) -- (R2.north);
\draw[predlink] (Q3.south) -- (R3.north);

\node[inode, label=right:{$\alpha_2$}] at (b1) {};
\node[inode, label=right:{$\alpha_1$}] (RootR) at (b2) {};

\node (OutR) at (2,1) {$\hat{y}_{\mathsf{T}_R}$};

\draw[rootlink] (RootR.south) -- (OutR.north);

\node at (2,0) {$\mathsf{T}_R=(1(23))$};
\end{scope}

\end{tikzpicture}

\caption{The two rooted planar binary trees for \(N\!=\!3\) and the ordered configuration \((\texttt{expert},\texttt{assistant},\texttt{AI})\). The trees represent two distinct protocols of the same prediction-task HAI: \(\mathsf T_L=((12)3)\) first combines expert and assistant predictions, whereas \(\mathsf T_R=(1(23))\) first combines assistant and AI predictions. 
}
\label{fig:N3_trees_horizontal}
\end{figure}

\begin{proposition}[Protocol-indifference locus for \(N\!=\!3\)]
\label{prop:N3_protocol_indifference}
Let \(\mathsf T_L\) and \(\mathsf T_R\) be equipped with local linear pooling and local parameters \((\alpha_1,\alpha_2)\in[0,1]^2\). Define $P_{\mathsf T_L,\mathsf T_R}(\alpha_1,\alpha_2)$ and the protocol-indifference locus $\mathcal S_3$ as in Proposition~\ref{prop:equality_locus}. Then

\begin{enumerate}[label=(\roman*),leftmargin=*]
\item $P_{\mathsf T_L,\mathsf T_R}=\sum_{0\le i,j\le2}
A_{ij}\alpha_1^i\alpha_2^j$ has coefficients such that $\sum_{0\le i,j\le2}
A_{ij}=0$.
\item If \(P_{\mathsf T_L,\mathsf T_R}(\alpha_1,\alpha_2)<0\), then
\(\mathsf T_L\) has higher complementarity than \(\mathsf T_R\). If \(P_{\mathsf T_L,\mathsf T_R}(\alpha_1,\alpha_2)>0\), then
\(\mathsf T_R\) has higher complementarity than \(\mathsf T_L\).
\item $(0,\alpha_2), (\alpha_1,1) \in\mathcal S_3$, where $\alpha_1,\alpha_2\in[0,1]$, since the two protocols produce identical outputs on these boundaries. 
\end{enumerate}
\end{proposition}
\begin{proof}
$(i)$ For \(N=3\), the recursive left-to-right ordering, the left tree keeps the parameter order \((\alpha_1,\alpha_2)\), while for the right tree \(\alpha_1\) is the root parameter and \(\alpha_2\) is the parameter of the internal node \((23)\). Writing $\hat y_{\mathsf T_L}(\alpha_1,\alpha_2)$ and $\hat y_{\mathsf T_R}(\alpha_1,\alpha_2)$ explicitly one obtains $P_{\mathsf T_L,\mathsf T_R}(\alpha_1,\alpha_2)=\sum_{i,j=0}^2 A_{ij}\alpha^i_1\alpha^j_2$. Finally, $\sum_{0\le i,j\le 2}A_{ij}=0$ because $\hat y^{\mathsf T_L}(1,1)=\hat y^{\mathsf T_R}(1,1)=\hat y^{(1)}$. $(ii)$ follows from the definition of $P_{\mathsf T_L,\mathsf T_R}$ and~\eqref{eq:compl_inv_geom}. $(iii)$ follows from \eqref{eq:N3_difference}. 
\end{proof}

\subsubsection{Numerical illustration of \(\mathcal S_3\) for \(N\!=\!3\) regression}\label{num_exp2}
We study when \(\mathsf T_L\) and \(\mathsf T_R\) induce the same complementarity value under linear local pooling by visualizing \(\mathcal S_3\) and studying the sign and zeros of \(P(\alpha_1,\alpha_2)\), as characterized in Proposition~\ref{prop:N3_protocol_indifference}. Using the held-out California Housing target vector and the same AI prediction vector as in the previous numerical illustration, we set \(\hat y^{(3)}=\hat y^{AI}\), fix a corrective synthetic expert as \(\hat y^{(1)}\), and vary the assistant prediction \(\hat y^{(2)}\) across four regimes: (i) weakly corrective; (ii) strongly corrective; (iii) weakly non-corrective; and (iv) strongly non-corrective. Here, `(non-)corrective' refers to the displacement of the assistant prediction from the fixed AI prediction, \(\hat y^{(2)}-\hat y^{AI}\), relative to the AI residual \(\hat y^{AI}-y\): corrective displacements move against the AI residual, whereas non-corrective displacements reinforce it. All technical details are contained in Appendix~\ref{ml_details_exp2}. Figure~\ref{fig:exp2_summary} shows the signed difference \(P(\alpha_1,\alpha_2)/n\) over the local-weight square. 
The thick black contour is \(\mathcal S_3\), where both trees attain the same complementarity value. In regions where \(P(\alpha_1,\alpha_2)<0\), \(\mathsf T_L\) has the smaller squared loss and hence higher complementarity; in regions where \(P(\alpha_1,\alpha_2)>0\), \(\mathsf T_R\) has higher complementarity. 
Across the four regimes, the nontrivial branches of \(\mathcal S_3\) shift as the assistant displacement from the fixed AI prediction changes from corrective to non-corrective. Thus, at fixed numerical local weights, tree topology can affect complementarity.

\begin{figure}[h!]
    \centering
    \includegraphics[width=0.8\textwidth]{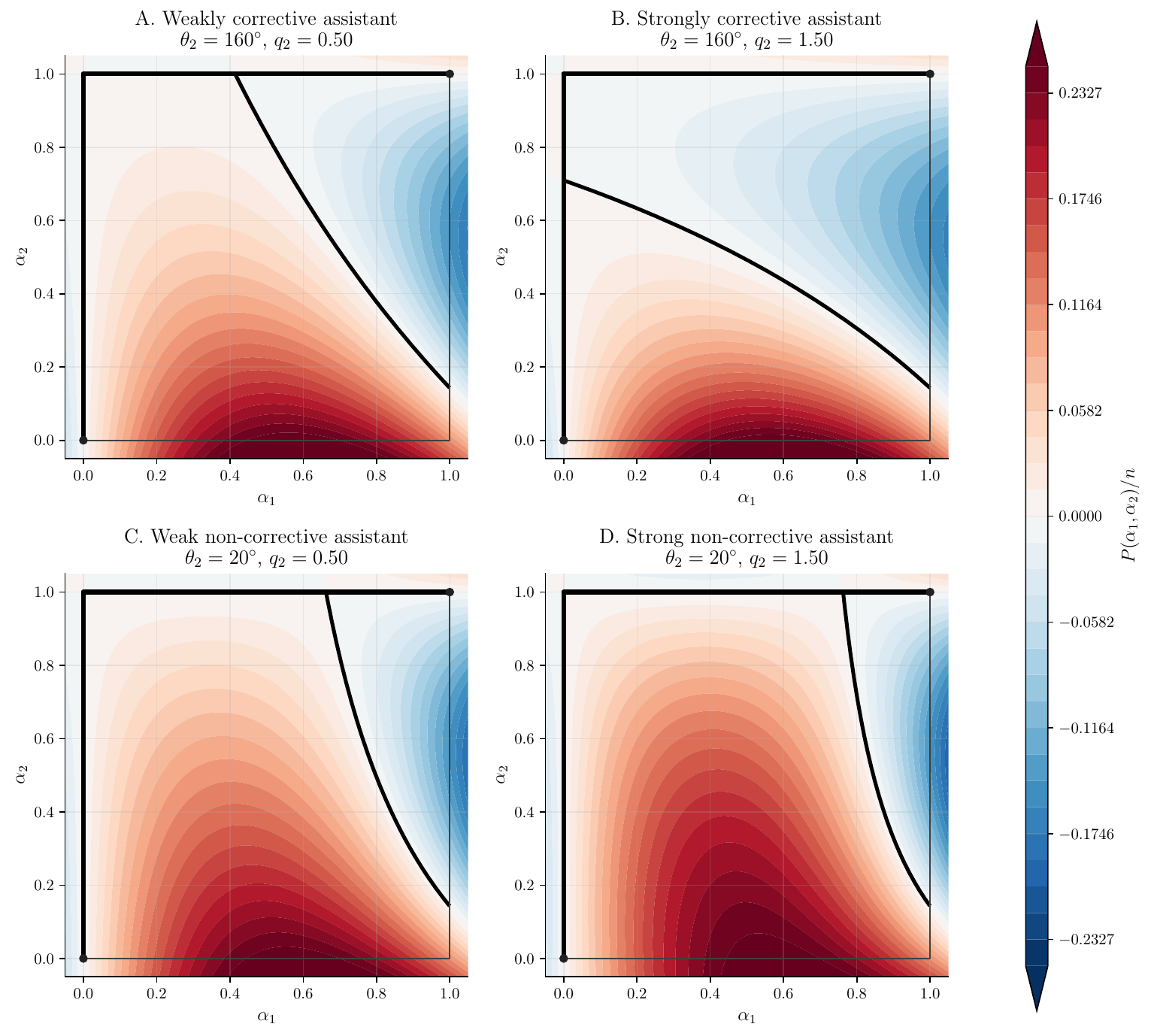}
\caption{\(N\!=3\!\) regression complementarity under linear pooling on the California housing dataset. Panels show the protocol-indifference function \(P(\alpha_1,\alpha_2)/n\) for \(\mathsf T_L=((12)3)\) and \(\mathsf T_R=(1(23))\); blue favors \(\mathsf T_L\), red favors \(\mathsf T_R\). Thick black contours mark \(P(\alpha_1,\alpha_2)=0\) inside \([0,1]^2\). The expert is fixed at \(\theta_1=180^\circ\), \(q_1=1.75\), while the assistant varies across four regimes.}
    \label{fig:exp2_summary}
\end{figure}

\section{Complementarity in Binary Classification}
\label{section:classification}
We now investigate complementarity in binary classification, with labels
\(y_i\in\{0,1\}=\mathcal Y\) and predicted probabilities
\(\hat y_i\in(0,1)=\hat{\mathcal Y}\), where \(\hat y_i\) denotes the predicted probability of class \(1\). We start with investigating reliance and complementarity.

\subsection{Reliance Never Reaches Complementarity}
\label{subsec:reliance}
In the \(N=2\) binary classification setting, reliance occurs when the human--AI interaction output is not a transformed or combined prediction, but simply one of the agent inputs: if \(\hat y_i^{HAI}=\hat y_i^{H}\), the case is one of \emph{self-reliance}; if \(\hat y_i^{HAI}=\hat y_i^{AI}\), it is one of \emph{AI reliance} \citep{schemmer_appr_reliance,zhang2020effect,ferrario2025being}.\footnote{In the HAI literature, the normative target is \emph{appropriate reliance}: the human follows correct AI advice when the AI is right and ignores it when the AI is wrong \citep{schemmer_appr_reliance}. This casewise notion directly extends to datasets measuring the proportion of samples on which the selected source is correct.} Reliance does not yield complementarity with pointwise min benchmark: 

\begin{theorem}[Selectors cannot yield complementarity with pointwise min benchmark]
\label{thm:selector_impossibility}
Let \(\mathsf{T}\in\Tcal_N\), and suppose the local rule \(m_2:~ (m_2(u,v))_i\in\{u_i,v_i\}\) for all \(u,v\in\Yhat^n,\ i=1,\dots,n.\) Then, for every dataset \(D\), $\Psi_{\mathsf{T},\mathrm{pw}}^{m_{\mathsf{T}}}(\hat y^{(1)},\dots,\hat y^{(N)};D)\le 0$. 
\end{theorem}

\begin{proof}
Fix a case \(i\).\ As \(m_2\) selects one of its two coordinate inputs,  recursive composition implies that the protocol output on coordinate \(i\) equals one of the leaf predictions: $(m_{\mathsf T}(\hat y^{(1)},\dots,\hat y^{(N)}))_i
\in\{\hat y_i^{(1)},\dots,\hat y_i^{(N)}\}$ and
\(
\min_{1\le j\le N}\ell(y_i,\hat y_i^{(j)})
\le
\ell\bigl(y_i,(m_{\mathsf T}(\hat y^{(1)},\dots,\hat y^{(N)}))_i\bigr).
\)
Averaging over \(i=1,\dots,n\) yields the result. 
\end{proof}

\begin{corollary}
Consider the two-agent configuration \((\textnormal{\texttt{human}},\textnormal{\texttt{AI}})\). If their interaction output exhibits self-reliance or AI reliance on every sample of \(D\), then the HAI cannot achieve complementarity relative to the pointwise-min benchmark, regardless of the proportion of appropriately reliant cases on \(D\).
\end{corollary}

\subsection{Endpoint-Monotone Losses}
\label{subsection:endpoint_monotone_loss}
We now turn our attention to a wide family of losses for classification problems.

\begin{definition}[Endpoint-monotone loss]
\label{def:endpoint_monotone_loss}
A loss function \(\ell:\Y\times\Yhat\to[0,\infty)\) is called
\emph{endpoint-monotone} if \(\hat y\mapsto \ell(0,\hat y)\) is nondecreasing and
\(\hat y\mapsto \ell(1,\hat y)\) is nonincreasing on \((0,1)\).
\end{definition}
Thus, under endpoint-monotonicity, when the true label is \(0\), assigning more probability to class \(1\) cannot decrease the loss; when the true label is \(1\), assigning more probability to class \(1\) cannot increase the loss. This condition is satisfied by the main probabilistic classification losses used in practice.

\begin{proposition}
\label{prop:endpoint_monotone}
The following loss functions are endpoint-monotone:
\begin{enumerate}[label=(\roman*),leftmargin=*]
    \item \(\ell_F(y,\hat y)=D_F(y,\hat y)\), where \(D_F\) is a Bregman divergence generated by a convex \(F\) on $[0,1]$, differentiable on $(0,1)$, and  with finite endpoint values     \citep{bregman1967relaxation};
    \item \(\ell_f(y,\hat y)=D_f(P_y\|Q_{\hat y})\), where \(D_f\) is an \(f\)-divergence 
    between the Bernoulli distributions of ground truth and predicted probabilities, with $f:[0,\infty)\rightarrow\R$ convex, finite on $[0,\infty)$, differentiable on $(0,\infty)$, and normalized by $f(1)=0$ \citep{ali1966general}.
\end{enumerate}
\end{proposition}
\begin{proof}
For the Bregman case, let \(F:[0,1]\to\mathbb R\) be convex and continuous on
\([0,1]\), and differentiable on \((0,1)\). For \(y\in\{0,1\}\) and
\(\hat y\in(0,1)\), define $D_F(y,\hat y):=F(y)-F(\hat y)-F'(\hat y)(y-\hat y)$.
We show endpoint monotonicity. Write \(\ell_F(0,\hat y)=F(0)-G_0(\hat y)\) and \(\ell_F(1,\hat y)=F(1)-G_1(\hat y)\), where \(G_0(\hat y):=F(\hat y)-\hat yF'(\hat y)\) and \(G_1(\hat y):=F(\hat y)+(1-\hat y)F'(\hat y)\). By convexity of $F$, \(G_0\) is nonincreasing and \(G_1\) is nondecreasing on \((0,1)\). Thus \(\ell_F=D_F\) is endpoint-monotone.

We now consider \(f\)-divergences. Under the standard convention
\[
D_f(P\|Q):=\sum_{k\in\{0,1\}} Q_k f\!\left(\frac{P_k}{Q_k}\right),
\]
let \(f:[0,\infty)\to\mathbb R\) be convex, differentiable on
\((0,\infty)\), normalized by \(f(1)=0\), and finite at \(0\). For binary classification, write the hard-label distribution and the predicted distribution as \(P_y=(y,1-y)\) and \(Q_p=(\hat y,1-\hat y)\). The induced hard-label
loss is
\[
\ell_f(y,\hat y):=D_f(P_y\|Q_{\hat y})
=
\hat y f\!\left(\frac{y}{\hat y}\right)
+
(1-\hat y)f\!\left(\frac{1-y}{1-\hat y}\right).
\]
Hence
\[
\ell_f(1,\hat y)=\hat y f\!\left(\frac1{\hat y}\right)+(1-\hat y)f(0),
\qquad
\ell_f(0,\hat y)=\hat y f(0)+(1-\hat y)f\!\left(\frac1{1-\hat y}\right).
\]
Set \(g(t):=f(t)-t f'(t)\). Differentiating the induced binary loss gives
\[
\frac{d}{d\hat y}\ell_f(1,\hat y)
=
g\!\left(\frac1{\hat y}\right)-f(0),
\qquad
\frac{d}{d\hat y}\ell_f(0,\hat y)
=
f(0)-g\!\left(\frac1{1-\hat y}\right).
\]
By convexity, the tangent inequality at \(t>0\), evaluated at \(0\), gives
\[
f(0)\ge f(t)+f'(t)(0-t)=f(t)-t f'(t)=g(t).
\]
Therefore \(g(1/\hat y)-f(0)\le0\) and \(f(0)-g(1/(1-\hat y))\ge0\). Hence \(\hat y\mapsto \ell_f(1,\hat y)\) is nonincreasing and \(\hat y\mapsto \ell_f(0,\hat y)\) is nondecreasing. Thus, \(\ell_f\) is endpoint-monotone.
\end{proof}

\begin{corollary}
\label{cor:standard_endpoint_monotone_losses}
The following loss functions are endpoint-monotone: the Brier loss, binary cross-entropy/Kullback-Leibler loss, squared Hellinger loss, Pearson \(\chi^2\) loss, Jensen--Shannon loss, and Tsallis \(f\)-divergence losses \citep{tsallis1988possible,furuichi2004fundamental}.
\end{corollary}
\begin{proof}
The Brier loss and binary cross-entropy are Bregman losses, generated respectively by \(F(\hat y)=\hat y^2\) and by the negative entropy
\(F(\hat y)=\hat y\log \hat y+(1-\hat y)\log(1-\hat y)\). 
The remaining examples are \(f\)-divergence losses, with $f:[0,\infty)\rightarrow \R$. Squared Hellinger is generated, up to normalization, by \(f(t)=\frac12(\sqrt t-1)^2\); Pearson
\(\chi^2\) by \(f(t)=(t-1)^2\); Jensen--Shannon by
\(
f(t)=\frac12\left[
t\log\frac{2t}{1+t}
+
\log\frac{2}{1+t}
\right];
\)
and Tsallis \(f\)-divergences, for \(q>0\), \(q\neq1\), by the normalized generator $f_q(t)=\frac{t^q-t}{q-1}$. These generators are convex, differentiable on \((0,\infty)\), normalized by \(f(1)=0\) and finite at \(0\). Proposition~\ref{prop:endpoint_monotone}
therefore applies.
\end{proof}

\subsection{An Impossibility Theorem for Complementarity in Binary Classification}
\label{subsec:classification_impossibility}
We now obtain a structural impossibility result for binary classification. The result derives from combining endpoint-monotonicity of loss functions and a local condition on interaction rules in the trees while using the pointwise-min benchmark.

\begin{definition}[Internality property of local rules]
\label{def:internal_rule}
A local rule $m_2:(0,1)^n\times(0,1)^n\to(0,1)^n$
satisfies the internality property if, for every \(u,v\in(0,1)^n\) and every coordinate
\(i\), $\min\{u_i,v_i\} \le (m_2(u,v))_i\le \max\{u_i,v_i\}$.
\end{definition}
Internality is a standard property of quasi-arithmetic means \citep{andrey1930notion,nagumo1930klasse}:
\begin{lemma}
\label{lem:quasi_internal}
Quasi-arithmetic means satisfy the internality property.
\end{lemma}

\begin{theorem}[Impossibility of complementarity in binary classification under pointwise-min benchmark]
\label{thm:endpoint_internal_impossibility}
Let \(N\!\ge\!2\), let \(\mathsf T\in\Tcal_N\), and let
\(D=\{(x_i,y_i)\}_{i=1}^n\) be a labeled dataset with \(y_i\in\{0,1\}\). Let $\hat y^{(1)},\dots,\hat y^{(N)}\in(0,1)^n$
be any collection of prediction vectors on \(D\). 
If \(m_2\) satisfies the internality property, for every endpoint-monotone binary loss \(\ell\) in Definition~\ref{def:endpoint_monotone_loss},
\[
\Psi_{\mathsf T,\mathrm{pw}}^{m_{\mathsf T}}
(\hat y^{(1)},\dots,\hat y^{(N)};D)
\le 0.
\]
\end{theorem}

\begin{proof}
We begin with an observation: the internality property is preserved through tree composition. Let \(m_2\) satisfy the internality property, let \(\mathsf{T}\in\Tcal_N\), and fix a case \(i\). If $\hat{y}_i^{(1)},\dots,\hat{y}_i^{(N)}\in(0,1)$ are the leaf probabilities on that case and \(\hat{y}_i^{\mathsf{T}}\) is the
output produced by \(\mathsf{T}\), then
\begin{equation}
\min_{1\le j\le N} \hat{y}_i^{(j)}
\le
\hat{y}_i^{\mathsf{T}}
\le
\max_{1\le j\le N} \hat{y}_i^{(j)}.
\label{eq:internality_tree}
\end{equation}
Proceed by induction on the number of leaves of \(\mathsf T\). The case of one leaf is immediate. For the inductive step, suppose the root of \(\mathsf T\)
combines two subtrees \(\mathsf T_L\) and \(\mathsf T_R\), with outputs \(u_i\)
and \(v_i\) on case \(i\). By the inductive hypothesis, each of \(u_i\) and
\(v_i\) lies in the interval spanned by the leaf probabilities
\(\hat{y}_i^{(1)},\dots,\hat{y}_i^{(N)}\). Since \(m_2\) is internal,
\(\hat{y}_i^\mathsf{T}=m_2(u_i,v_i)\) lies between \(u_i\) and \(v_i\), hence
also in the interval spanned by the leaves. Then, for every case \(i\), $\min_{1\le j\le N} \hat{y}_i^{(j)}
\le \hat{y}_i^{\mathsf T} \le \max_{1\le j\le N} \hat{y}_i^{(j)}$. If \(y_i=0\), endpoint monotonicity gives
\[
\ell(0,\hat y_i^{\mathsf T})
\ge
\ell\!\left(0,\min_{1\le j\le N}\hat y_i^{(j)}\right)
=
\min_{1\le j\le N}\ell(0,\hat y_i^{(j)}).
\]
If \(y_i=1\), endpoint monotonicity gives
\[
\ell(1,\hat y_i^{\mathsf T})
\ge
\ell\!\left(1,\max_{1\le j\le N}\hat y_i^{(j)}\right)
=
\min_{1\le j\le N}\ell(1,\hat y_i^{(j)}).
\]
Therefore, for every \(i=1,\dots,n\), we have that $\min_{1\le j\le N}\ell(y_i,\hat y_i^{(j)}) \le \ell(y_i,\hat y_i^{\mathsf T})$. Averaging over \(i\) gives
\[
\frac1n\sum_{i=1}^n
\min_{1\le j\le N}\ell(y_i,\hat y_i^{(j)})
\le
\frac1n\sum_{i=1}^n
\ell(y_i,\hat y_i^{\mathsf T})\Longleftrightarrow
\Psi_{\mathsf T,\mathrm{pw}}^{m_{\mathsf T}}
(\hat y^{(1)},\dots,\hat y^{(N)};D)
\le 0. 
\]
\end{proof}

\begin{corollary}
In binary classification, complementarity under pointwise-min benchmark  cannot be achieved by any internal local rule when the loss is one of the Bregman or finite Bernoulli $f$-divergence losses covered by Proposition~\ref{prop:endpoint_monotone}. 
\end{corollary}
Since \(\Psi_{\mathsf T,\mathrm{ba}}^{m_{\mathsf T}}=\Psi_{\mathsf T,\mathrm{pw}}^{m_{\mathsf T}}+\Delta_{\mathrm{ba,pw}},\)---see \eqref{eq:gap}---Theorem~\ref{thm:endpoint_internal_impossibility} implies \(\Psi_{\mathsf T,\mathrm{ba}}^{m_{\mathsf T}}\le\Delta_{\mathrm{ba,pw}}.\) Thus, any  complementarity gain under the best-average benchmark in binary classification is bounded by the benchmark gap.

\begin{remark}[Multiclass cross-entropy]
\label{rem:multiclass_cross_entropy_internality}
The same obstruction holds in multiclass problems under cross-entropy. Let \(K\ge2\), let each agent output probabilities
\(\hat y_i^{(j)}\in\operatorname{int}(\Delta^{K-1})\), and let \(\bar k(i)\) denote the true class of sample \(i\). For multiclass cross-entropy,
\(\ell(y_i,\hat y_i)=-\log \hat y_{i,\bar k(i)}\), the tree-relative
complementarity functional can be written as
\[
n\Psi_{\mathsf T,\mathrm{pw}}^{m_{\mathsf T}}
=
\sum_{i=1}^n
\log
\left(
\frac{
\hat y^{\mathsf T}_{i,\bar k(i)}
}{
\max_{1\le j\le N}\hat y^{(j)}_{i,\bar k(i)}
}
\right).
\]
Let $m_2:\operatorname{int}(\Delta^{K-1})\times
\operatorname{int}(\Delta^{K-1})
\to
\operatorname{int}(\Delta^{K-1})$ be a local rule.
If $m_2$ satisfies the internality property coordinatewise, i.e., for every class \(k\), $(m_2(u,v))_k\in[\min\{u_k,v_k\},\\\max\{u_k,v_k\}]$, then 
\(\hat y^{\mathsf T}_{i,\bar k(i)}
\le
\max_{1\le j\le N}\hat y^{(j)}_{i,\bar k(i)}\) for every \(i\), and hence
\(\Psi_{\mathsf T,\mathrm{pw}}^{m_{\mathsf T}}\le0\). 
\end{remark}
In summary, Theorem~\ref{thm:endpoint_internal_impossibility} identifies two sources of obstruction to complementarity in binary classification: endpoint-monotone losses combined with internal local rules cannot outperform the pointwise best available leaf prediction. Hence, achieving complementarity in binary classification under the pointwise-min benchmark requires relaxing at least one of these two conditions. Relaxing endpoint monotonicity seems to be conceptually unattractive, since it encodes a natural requirement of loss functions in binary classification. The alternative is therefore to relax internality. A possibility to do this is to amplify a quasi-arithmetic logit score before mapping it back to probability space. For instance, the amplified logit rule becomes
\begin{equation}
(m_{2,\alpha,\lambda}^{\mathrm{logit}}(u,v))_i
:=
\sigma\!\Bigl(
\lambda\bigl[
\alpha\operatorname{logit}(u_i)
+
(1-\alpha)\operatorname{logit}(v_i)
\bigr]
\Bigr),
\label{eq:amplified_logit}
\end{equation}
where \(\sigma(t)=(1+e^{-t})^{-1}, t\in\R\) and $\lambda \ge 1$. In HAI terms, amplification moves the  combined prediction of two agents away from the zero in logit space, or equivalently away from probability \(\sigma(0)=\frac{1}{2}\), in the direction selected by the agents' pooled log-odds predictions. Note that the rule $m_{2,\alpha,\lambda}^{\mathrm{logit}}$ is normalized logarithmic pooling of Bernoulli forecasts \citep{neyman2023no,neyman2023proper} if  $\lambda=1$. For \(N=2\), the effect of amplifying the logit rule on complementarity is explicitly derived. Write 
\[
z_i(\alpha):=\alpha\operatorname{logit}(\hat y_i^{(1)})+(1-\alpha)\operatorname{logit}(\hat y_i^{(2)}),
\qquad
m_i^{\alpha,\lambda}:=\sigma(\lambda z_i(\alpha)).
\]
Under binary cross-entropy, with \(I_1=\{i:y_i=1\}\) and \(I_0=\{i:y_i=0\}\), the amplified logit rule gives the complementarity functional
\begin{align*}
n\Psi_{\mathsf T,\mathrm{pw}}^{\alpha,\lambda}(\hat y_i^{(1)},\hat y_i^{(2)};D)
=
&\sum_{i\in I_1}
\left[
-\log\max\{\hat y_i^{(1)},\hat y_i^{(2)}\}
-
\log(1+e^{-\lambda z_i(\alpha)})
\right]
+\\
&\sum_{i\in I_0}
\left[
-\log\max\{1-\hat y_i^{(1)},1-\hat y_i^{(2)}\}
-
\log(1+e^{\lambda z_i(\alpha)})
\right].
\end{align*}
Thus, for \(i\in I_1\), the \(i\)-th summand is positive if and only if
\((m_{2,\alpha,\lambda}^{\mathrm{logit}}(\hat y^{(1)},\hat y^{(2)}))_i>\max\{\hat y_i^{(1)},\hat y_i^{(2)}\}\). For \(i\in I_0\), it is positive if and only if
\((m_{2,\alpha,\lambda}^{\mathrm{logit}}(\hat y^{(1)},\hat y^{(2)}))_i<\min\{\hat y_i^{(1)},\\\hat y_i^{(2)}\}\). For these local complementarity cases the non-internal output $m_{2,\alpha,\lambda}^{\mathrm{logit}}(\hat y^{(1)},\hat y^{(2)})$ must move outside the interval spanned by the two input probabilities, and it must do so in the direction of the true class. We analyze this point numerically in Appendix~\ref{ml_details_exp3}.

\section{Discussion and Conclusions}
\label{sec:discussion_conclusions}
Real-world human--AI interactions often involve several human professionals and AI systems contributing to complex decision-making processes. In such settings, predictive tasks may constitute subprocedures within broader workflows involving deliberation, multiple knowledge bases, decisions, and actions. Developing a mathematical account of complementarity in these predictive subprocedures is a first step toward understanding multi-agent HAIs and  identifying protocols that optimize predictive performance. Our framework provides such an account by representing interactions through ordered agent-role configurations and rooted planar binary trees, with local aggregation rules recursively composing agent predictions into a final output. This formalization extends complementarity beyond the traditional two-agent setting while separating protocol structure, aggregation mechanisms, evaluation losses, and performance benchmarks.

A first consequence is that complementarity is relative to the benchmark against which the interaction is evaluated. The best-average benchmark measures improvement over the best fixed agent across a dataset, whereas the pointwise-min benchmark measures improvement over the best available prediction on each case. The latter is particularly relevant when evaluating predictive subprocedures in high-risk settings by imposing a stricter aggregate comparison. The benchmark gap $\Delta_{\mathrm{ba,pw}}$ quantifies the difference between the two evaluations. In particular, selector-based protocols, including self- and AI-reliance, cannot achieve complementarity relative to the pointwise-min benchmark, although they may appear complementary under the best-average benchmark. Appropriate reliance may still be desirable for accountability, usability, or safety in simplified or low-stakes HAI settings \citep{schemmer_appr_reliance}.

A second contribution concerns the role of protocol topology. Under unconstrained linear pooling, every protocol tree provides a bijective coordinate system for the same simplex of global leaf weights. Consequently, any two trees with the same ordered leaves admit a unique reparameterization preserving these weights, their root outputs, and complementarity values. Tamari rotations provide explicit local transformations implementing this universal equivalence, with path-independent compositions. Thus, protocol topology alone does not restrict attainable predictive performance under unrestricted linear pooling. Nevertheless, different trees may produce different
outputs when evaluated at identical local parameters, as illustrated by the $N=3$ fixed-parameter analysis. These results provide a basis for examining the constraints that characterize real-world HAIs. Professional hierarchies, organizational requirements, and epistemically justified confidence may restrict admissible local aggregation weights. Under such restrictions, an output-preserving reparameterization need not preserve admissibility. This is important for the design and evaluation of HAI protocols subject to requirements, such as operational costs, including interaction depth, communication, fairness, robustness, and monitoring burden \citep{ferrario2026epistemology}.

Third, our task-specific results further clarify the conditions under which complementarity can be achieved. For regression under squared loss, maximizing complementarity corresponds to minimizing the Euclidean distance between the protocol output and the ground-truth vector. In the two-agent linear-pooling case, the optimal weight is determined by projection onto the segment joining the agents' prediction vectors. Consequently, the direction and magnitude of disagreement determine whether aggregation can correct prediction errors. For binary classification, by contrast, internal aggregation rules cannot
achieve complementarity relative to the pointwise-min benchmark under endpoint-monotone losses, including the Bregman and finite Bernoulli $f$-divergence losses considered here. Our classification impossibility theorem is a negative that complements the regime-based analysis of \citep{Donahue2022FAccT}, under different assumptions on aggregation rules, losses, and performance benchmarks. Thus, if one accepts the pointwise-min benchmark as the relevant standard for complementarity in certain high-stakes HAI domains, then many common human--AI classification protocols cannot achieve complementarity by interpolating between human and AI probabilities \citep{Vaccaro2024NHB}. An analogous obstruction applies to multiclass cross-entropy. Non-internal mechanisms, such as amplified logit pooling, may escape these restrictions, although their reliability and suitability require further investigation.

\vspace{1em}

We foresee two lines of future work. First, the formalization should be brought closer to the requirements of responsible AI research by adding costs for interaction depth, monitoring burden, and protocol complexity, in the spirit of \emph{efficient complementarity} \citep{ferrario2026epistemology}; incorporating constraints from fairness, robustness, oversight, documentation, and workflow feasibility within tree topology; studying stability under dataset shift and temporal change; and testing the framework on real human--AI and multi-agent prediction settings. As our framework abstracts from adaptive information exchange, future work should examine how these prediction-task subroutines can be embedded in adaptive, preference-sensitive, and organizationally constrained interactions, including the effects of temporal change, communication, oversight, and protocol costs.

Second, a mathematical direction is to study the present construction as a
theory of \emph{decorated planar binary trees} \citep{zhang2020hopf}. The Tamari reparameterization formula~\eqref{eq:tamari_reparam} in Theorem~\ref{thm:tamari_rotation_invariance} can be read as a tree-level version of the skew-associativity identity familiar from barycentric algebras \citep{smith2011modes,fritz2009convex,zamojska2024barycentric}: reassociating a convex combination requires a corresponding rational transformation of the local weights. Our results show that these local transformations preserve the root prediction and therefore the complementarity functional. This suggests studying complementarity together with the algebraic structures on planar binary trees, including operadic grafting and the Hopf algebra of planar binary trees \citep{loday1998hopf,giraudo2018nonsymmetric,loday2012algebraic}. The resulting problem is to identify which tree operations, decorations, and parameter transports are complementarity-invariant, and which change the complementarity functional in a controlled way.

\section*{Acknowledgments}
This work was partly conducted within the framework of the EUonAIR Centre of Excellence in Responsible AI and Education. It was partially supported by a grant from Movetia, funded by the
Swiss Confederation. We also thank Alessandro Facchini and Matteo Casserini for many useful conversations on complementarity and human--AI interactions.

\newpage
\appendix

\section{The Pentagon Identity on the \(N\!=\!4\) Associahedron}
\label{stasheff_pentagon}
\begin{proof}
To verify Corollary~\ref{cor:stasheff_pentagon}, we use the five trees
$\mathsf T_1=(((12)3)4)$, $\mathsf T_2=((1(23))4)$,
$\mathsf T_3=((12)(34))$,  $\mathsf T_4=(1((23)4))$,
$\mathsf T_5=(1(2(34)))$, which form the vertices of the Stasheff associahedron \(\mathsf{Assoc}_4\); see Figure~\ref{fig:stasheff_pentagon_N4}. Throughout, we write $\boldsymbol{\alpha} =(\alpha_1,\alpha_2,\alpha_3)\in(0,1)^3$ for the parameterization of the `left comb' \(\mathsf T_1\), where \(\alpha_1\) labels the node \((12)\), \(\alpha_2\) labels the node \(((12)3)\), and \(\alpha_3\) labels the root.

\subsection{Path I: \(\mathsf T_1\to\mathsf T_3\to\mathsf T_5\)}

First consider the cover
\[
\mathsf T_1=(((12)3)4)\longrightarrow \mathsf T_3=((12)(34)).
\]
This rotates the fragment \(((AB)C)\) with
\[
A=(12),\qquad B=3,\qquad C=4.
\]
Hence, using the definition of the map $\Gamma$, the parameters on \(\mathsf T_3\) are
\[
\Gamma_{1,3}(\alpha_1,\alpha_2,\alpha_3)
=
\left(
\alpha_1,\,
\alpha_2\alpha_3,\,
\frac{(1-\alpha_2)\alpha_3}{1-\alpha_2\alpha_3}
\right).
\]
Here the coordinates on \(\mathsf T_3\) are ordered as: parameter at \((12)\), parameter at the root \(((12)(34))\), and parameter at \((34)\). Next consider the cover
\[
\mathsf T_3=((12)(34))\longrightarrow \mathsf T_5=(1(2(34))).
\]
This rotates the root fragment
\[
((AB)C)\longrightarrow (A(BC))
\]
with
\[
A=1,\qquad B=2,\qquad C=(34).
\]
If we write
\[
(\gamma_1,\gamma_2,\gamma_3)
=
\Gamma_{1,3}(\alpha_1,\alpha_2,\alpha_3),
\]
then
\[
\Gamma_{3,5}(\gamma_1,\gamma_2,\gamma_3)
=
\left(
\gamma_1\gamma_2,\,
\frac{(1-\gamma_1)\gamma_2}{1-\gamma_1\gamma_2},\,
\gamma_3
\right).
\]
Substituting the values of \(\gamma_1,\gamma_2,\gamma_3\) gives
\begin{align}
(\Gamma_{3,5}\circ\Gamma_{1,3})(\alpha_1,\alpha_2,\alpha_3)
= 
\left(
\alpha_1\alpha_2\alpha_3,
\frac{(1-\alpha_1)\alpha_2\alpha_3}{1-\alpha_1\alpha_2\alpha_3},
\frac{(1-\alpha_2)\alpha_3}{1-\alpha_2\alpha_3}
\right).
\label{eq:upper_path_transport}
\end{align}

\subsection{Path II: \(\mathsf T_1\to\mathsf T_2\to\mathsf T_4\to\mathsf T_5\)}

Now consider the cover
\[
\mathsf T_1=(((12)3)4)\longrightarrow \mathsf T_2=((1(23))4).
\]
This rotates the fragment
\[
((12)3)\longrightarrow (1(23)),
\]
so \(A=1\), \(B=2\), \(C=3\). Therefore
\[
\Gamma_{1,2}(\alpha_1,\alpha_2,\alpha_3)
=
\left(
\alpha_1\alpha_2,
\frac{(1-\alpha_1)\alpha_2}{1-\alpha_1\alpha_2},
\alpha_3
\right).
\]
Here the coordinates on \(\mathsf T_2\) are ordered as: parameter at \((1(23))\), parameter \((23)\), and parameter at the root. Next consider the cover
\[
\mathsf T_2=((1(23))4)\longrightarrow \mathsf T_4=(1((23)4)).
\]
This rotates the root fragment with
\[
A=1,\qquad B=(23),\qquad C=4.
\]
If
\[
(\delta_1,\delta_2,\delta_3)
=
\Gamma_{1,2}(\alpha_1,\alpha_2,\alpha_3),
\]
then
\[
\Gamma_{2,4}(\delta_1,\delta_2,\delta_3)
=
\left(
\delta_1\delta_3,
\delta_2,
\frac{(1-\delta_1)\delta_3}{1-\delta_1\delta_3}
\right).
\]
Substituting gives
\begin{align}
(\eta_1,\eta_2,\eta_3)
:=
(\Gamma_{2,4}\circ\Gamma_{1,2})(\alpha_1,\alpha_2,\alpha_3)
= \left(
\alpha_1\alpha_2\alpha_3,
\frac{(1-\alpha_1)\alpha_2}{1-\alpha_1\alpha_2},
\frac{(1-\alpha_1\alpha_2)\alpha_3}{1-\alpha_1\alpha_2\alpha_3}
\right).
\end{align}
Finally, consider the cover
\[
\mathsf T_4=(1((23)4))\longrightarrow \mathsf T_5=(1(2(34))).
\]
This rotates the fragment
\[
((23)4)\longrightarrow (2(34)),
\]
so \(A=2\), \(B=3\), \(C=4\). Therefore
\[
\Gamma_{4,5}(\eta_1,\eta_2,\eta_3)
=
\left(
\eta_1,
\eta_2\eta_3,
\frac{(1-\eta_2)\eta_3}{1-\eta_2\eta_3}
\right).
\]
Then, it follows that
\begin{align*}
&1-\eta_2=\frac{(1-\alpha_2)}{1-\alpha_1\alpha_2},
\\
&\eta_2\eta_3=
\frac{(1-\alpha_1)\alpha_2\alpha_3}{1-\alpha_1\alpha_2\alpha_3},\\
&1-\eta_2\eta_3
=\frac{1-\alpha_2\alpha_3}{1-\alpha_1\alpha_2\alpha_3},\\
&\frac{(1-\eta_2)\eta_3}{1-\eta_2\eta_3}=\frac{(1-\alpha_2)\alpha_3}{1-\alpha_2\alpha_3}.
\end{align*}
Collecting terms,
\begin{align}
(\Gamma_{4,5}\circ\Gamma_{2,4}\circ\Gamma_{1,2})(\alpha_1,\alpha_2,\alpha_3)
= 
\left(
\alpha_1\alpha_2\alpha_3,
\frac{(1-\alpha_1)\alpha_2\alpha_3}{1-\alpha_1\alpha_2\alpha_3},
\frac{(1-\alpha_2)\alpha_3}{1-\alpha_2\alpha_3}
\right).
\label{eq:lower_path_transport}
\end{align}
Comparing \eqref{eq:upper_path_transport} and \eqref{eq:lower_path_transport} proves
\[
(\Gamma_{3,5}\circ\Gamma_{1,3})(\alpha_1,\alpha_2,\alpha_3)
=
(\Gamma_{4,5}\circ\Gamma_{2,4}\circ\Gamma_{1,2})(\alpha_1,\alpha_2,\alpha_3),
\]
which is the pentagon identity. This establishes the commutativity of the \(N=4\) Stasheff pentagon at the level of reparameterizations. By Theorem~\ref{thm:tamari_rotation_invariance}, every edgewise transport preserves the root output. Hence both pathwise compositions preserve the root output from \(\mathsf T_1\) to \(\mathsf T_5\), and therefore
\[
\hat y_{\mathsf T_1}(\boldsymbol{\alpha})
=
\hat y_{\mathsf T_5}\bigl((\Gamma_{3,5}\circ\Gamma_{1,3})(\boldsymbol{\alpha})\bigr)
=
\hat y_{\mathsf T_5}\bigl((\Gamma_{4,5}\circ\Gamma_{2,4}\circ\Gamma_{1,2})(\boldsymbol{\alpha})\bigr).
\]
Regression complementarity invariance then follows from Equation~\eqref{eq:tree_functional}, verifying Corollary~\ref{cor:stasheff_pentagon}.
\end{proof}


\section{Machine-Learning Details for Numerical Illustration~1}
\label{ml_details_exp1}
Numerical illustration~1 illustrates the two-agent regression geometry from Proposition~\ref{prop:weighted_regression}. The purpose of the numerical illustration is to verify how the location of the constrained maximizer \(\alpha^\ast\) and the sign of complementarity depend on the angle and norm of the human--AI disagreement direction.

\subsection{Dataset and AI model}
We use the California Housing regression dataset from \texttt{scikit-learn}. The target is the median house value of California districts expressed in hundreds of thousands of dollars (\$100,000). The data are split into training and test sets with a \(75/25\) split. The AI predictor is a random-forest regressor with \(300\) trees, \(\texttt{min\_samples\_leaf}=2\). No hyperparameter tuning is performed, because the model is used only to generate a fixed AI prediction vector \(\hat y^{AI}\). Complementarity is computed on the held-out test set (\(n=5160\)), not on the training set. The resulting AI test RMSE is \(0.5102\),
with
\[
C_n=\|\hat y^{AI}-y\|_2^2=1343.2033 \quad ([\$100,000]^2).
\]

\subsubsection{Synthetic human predictions}
Starting from the fixed test-set vectors \(y\) and \(\hat y^{AI}\), we generate
synthetic human predictions with controlled geometry. Let
\[
e_1=\frac{\hat y^{AI}-y}{\|\hat y^{AI}-y\|_2}.
\]
We sample a random Gaussian vector, orthogonalize it against \(e_1\), and normalize it to obtain a unit vector \(e_2\) with
\(\langle e_1,e_2\rangle\simeq 0\). For a prescribed angle
\(\theta\) and norm ratio \(q\), the human--AI displacement is
\[
\hat y^H-\hat y^{AI}
=
q\|\hat y^{AI}-y\|_2
\bigl(
\cos(\theta)e_1+\sin(\theta)e_2
\bigr),
\]
so that
\[
\angle(\hat y^H-\hat y^{AI},\,\hat y^{AI}-y)=\theta,
\qquad
\frac{\|\hat y^H-\hat y^{AI}\|_2}{\|\hat y^{AI}-y\|_2}=q .
\]
The human prediction is then \(\hat y^H=\hat y^{AI}+(\hat y^H-\hat y^{AI})\). This construction fixes the AI model and the test labels, while varying only the geometry of the synthetic human prediction.

\subsection{Complementarity computation}
For each scenario, we compute \(A_n=\|\hat y^H-\hat y^{AI}\|_2^2\), \(B_n=\langle \hat y^H-\hat y^{AI},\hat y^{AI}-y\rangle\), \(C_n=\|\hat y^{AI}-y\|_2^2\), and the pointwise-oracle benchmark
\[
nK_{n,\mathrm{pw}}=
\sum_{i=1}^n
\min\{(\hat y_i^H-y_i)^2,(\hat y_i^{AI}-y_i)^2\}.
\]
For the linear team prediction $\hat y_{\mathsf T}(\alpha)=\alpha \hat y^H+(1-\alpha)\hat y^{AI}$,
 $\alpha\in[0,1]$, the plotted quantity is
\[
n\Psi(\alpha)
:=
-A_n\alpha^2-2B_n\alpha+(nK_{n,\mathrm{pw}}-C_n),
\]
with constrained maximizer
\[
\alpha^\ast
=
\Pi_{[0,1]}\!\left(-\frac{B_n}{A_n}\right).
\]

\subsection{Scenarios}
We evaluate five controlled scenarios:
\[
(\theta,q)\in
\{(0,1.0),(\frac{\pi}{2},1.0),(\frac{3\pi}{4},0.5),
(\frac{3\pi}{4},2.5),(\pi,1.25)\}.
\]
The first two scenarios test non-corrective and orthogonal human movement. The third is corrective but too short, so the unconstrained maximizer lies outside the feasible segment. The last two scenarios satisfy the interior condition and achieve positive complementarity. The \(\theta=\pi\) case is the collinear corrective case: the human--AI line passes through the ground-truth direction, so the geometric loss term proportional to \(\sin^2\theta\) vanishes.

\section{Numerical Illustration 2 implementation details}
\label{ml_details_exp2}
We use the California Housing regression dataset with the same preprocessing, train--test split, and random forest model as in Numerical Illustration~1.\ A random forest regressor with \(300\) trees and \texttt{min\_samples\_leaf=2} is trained on the training set. All equality loci and performance quantities are computed on the held-out test set. Let \(\hat y^{AI}\) denote the AI prediction vector on the test set and let
\[
r_{AI}:=\hat y^{AI}-y
\]
be the AI residual vector. We construct synthetic agent predictions in an
AI-anchored geometry. For each synthetic agent \(j\), let
\[
d_j:=\hat y^{(j)}-\hat y^{AI}
\]
be its displacement from the AI prediction. We control the angle
\(\theta_j=\angle(d_j,r_{AI})\) and the relative displacement size
\[
q_j:=\frac{\|\hat y^{(j)}-\hat y^{AI}\|_2}{\|\hat y^{AI}-y\|_2}
=
\frac{\|d_j\|_2}{\|r_{AI}\|_2}.
\]
Thus, values \(\theta_j>90^\circ\) indicate displacements that are corrective relative to the AI residual, whereas values \(\theta_j<90^\circ\) move broadly in the same direction as the AI residual.

To generate the synthetic predictions, we draw a fixed unit vector
\(u_\perp\) orthogonal to \(r_{AI}\) and set
\[
\hat y^{(j)}
=
\hat y^{AI}
+
q_j\|r_{AI}\|_2
\left(
\cos(\theta_j)\frac{r_{AI}}{\|r_{AI}\|_2}
+
\sin(\theta_j)u_\perp
\right).
\]
In all scenarios, \(\hat y^{(1)}\) is a fixed corrective expert with
\(\theta_1=180^\circ\) and \(q_1=1.75\), while \(\hat y^{(3)}=\hat y^{AI}\). The assistant \(\hat y^{(2)}\) is varied across
four regimes:
\[
(\theta_2,q_2)\in
\{(160^\circ,0.50),\ (160^\circ,1.50),\
(20^\circ,0.50),\ (20^\circ,1.50)\}.
\]
These correspond respectively to weakly corrective, strongly corrective,
weakly non-corrective, and strongly non-corrective assistant displacements
relative to the AI residual.

For each scenario, we evaluate the two \(N=3\) protocol trees $\mathsf T_L=((12)3)$, $\mathsf T_R=(1(23))$, with $\hat y^{\mathsf T_L}$ and
$\hat y^{\mathsf T_R}$ given as follows:

\begin{align*}
\hat y^{\mathsf T_L}(\alpha_1,\alpha_2)
&=\alpha_1\alpha_2\hat y^{(1)}+(1-\alpha_1)\alpha_2\hat y^{(2)}+
(1-\alpha_2)\hat y^{(3)}, \\
\hat y^{\mathsf T_R}(\alpha_1,\alpha_2)
&=\alpha_1\hat y^{(1)}+(1-\alpha_1)\alpha_2\hat y^{(2)}+
(1-\alpha_1)(1-\alpha_2)\hat y^{(3)}.
\end{align*}
The plotted quantity is the signed test-set MSE difference
\begin{align*}
\frac{P(\alpha_1,\alpha_2)}{n}
=\frac{1}{n}\left(\|y-\hat y^{\mathsf T_L}\|_2^2
-\|y-\hat y^{\mathsf T_R}\|_2^2\right)
=\operatorname{MSE}(\mathsf T_L)-\operatorname{MSE}(\mathsf T_R).
\end{align*}
Hence \(P/n=0\) is the protocol-indifference locus \(\mathcal S_3\), \(P/n<0\) indicates lower MSE for \(\mathsf T_L\), and \(P/n>0\) indicates lower MSE for \(\mathsf T_R\). Here, lower tree MSE corresponds to higher complementarity value for that tree---see eq.~\eqref{eq:squared_loss_geometry}. The two root outputs satisfy \(\hat y^{\mathsf T_L}-\hat y^{\mathsf T_R}=\alpha_1(1-\alpha_2)(\hat y^{(3)}-\hat y^{(1)})\). Consequently, \(\alpha_1=0\) and \(\alpha_2=1\) are structural branches of \(\mathcal S_3\) in every scenario. In the implementation, \(P/n\) is evaluated on a \(251\times 251\) grid over \([0,1]^2\).

\section{Numerical Illustration 3 implementation details}
\label{ml_details_exp3}
Numerical illustration~3 studies binary classification under cross-entropy and amplified logit pooling in the \(N=2\) case. The goal is not to fit a real-data classifier, but to illustrate the mechanism by which leaving internality can make positive complementarity possible. We fix \(\alpha=0.5\) and evaluate several amplification levels \(\lambda\in\{1,2,5,10,20\}\).

We generate balanced $n=1000$ binary labels \(y_i\in\{0,1\}\), with
\(I_1=\{i:y_i=1\}\) and \(I_0=\{i:y_i=0\}\). For each simulated pair of
probabilistic predictors, the construction is performed in logit space. Let
\(s_i:=2y_i-1\in\{-1,1\}\). We first generate a midpoint logit \(C_i\), which is
the unamplified pooled logit for \(\alpha=\frac{1}{2}\). Its magnitude is sampled around
a dataset-level signal-strength parameter, while its sign is \(s_i\) except for
a randomly flipped fraction of observations. Thus most pooled logits point
toward the correct class, while a controlled fraction points in the wrong
direction, meaning that amplification
pushes the output opposite to the locally complementary direction: toward class
\(0\) on \(I_1\), or toward class \(1\) on \(I_0\). The two agent logits are then placed symmetrically around this
midpoint as \(z_i^{(1)}=C_i+\delta_i/2\) and
\(z_i^{(2)}=C_i-\delta_i/2\), where \(\delta_i\) is Gaussian disagreement noise.
The resulting probabilities are
\(\hat y_i^{(1)}=\sigma(z_i^{(1)})\) and \(\hat y_i^{(2)}=\sigma(z_i^{(2)})\).
For a fixed \(\lambda\), the amplified logit-pooling output is
\(m_{2,\frac{1}{2},\lambda}^{\mathrm{logit}}(\hat y_i^{(1)},\hat y_i^{(2)})=\sigma(\lambda C_i)\), since \(\alpha=\frac{1}{2}\). We then compute the
pointwise-oracle complementarity functional
\(n\Psi=\sum_i[\min\{\ell(y_i,\hat y_i^{(1)}),\ell(y_i,\hat y_i^{(2)})\}
-\ell(y_i,m_{2,\frac{1}{2},\lambda}^{\mathrm{logit}}(\hat y_i^{(1)},\hat y_i^{(2)}))]\), where \(\ell\) is binary cross-entropy. Thus
\(n\Psi>0\) means that the amplified pooled prediction beats the pointwise best
input prediction on average.
We also compute the class-wise rates of canonical local complementarity. On
\(I_1\), we define \(k_1\) as the fraction of observations satisfying
\(m_{2,\frac{1}{2},\lambda}^{\mathrm{logit}}(\hat y_i^{(1)},\hat y_i^{(2)})>\max\{\hat y_i^{(1)},\hat y_i^{(2)}\}\). On \(I_0\), we define \(k_0\) as
the fraction satisfying \(m_{2,\frac{1}{2},\lambda}^{\mathrm{logit}}(\hat y_i^{(1)},\hat y_i^{(2)})<\min\{\hat y_i^{(1)},\hat y_i^{(2)}\}\). These are
the local cases in which the amplified output leaves the interval spanned by the
two input probabilities in the direction of the true class. For \(\lambda=1\),
ordinary logit pooling is internal, so these strict outside-interval conditions
cannot occur. For \(\lambda>1\), amplification may create such cases by moving
the pooled logit away from \(0\), toward class \(1\) if \(C_i>0\) and toward
class \(0\) if \(C_i<0\). We finalize the numerical illustration computing the complementarity functional using binary cross-entropy as loss function. We show results in Figure~\ref{fig:exp3_simulation}.
For each simulated pair, we compute the global complementarity value \(n\Psi\) and the class-wise rates \(k_0\) and \(k_1\) of canonical local complementarity: on \(I_0\), the pooled prediction lies below both input probabilities, while on \(I_1\), it lies above both input probabilities. At \(\lambda=1\), ordinary internal logit pooling cannot satisfy these strict outside-interval conditions, in line with Theorem~\ref{thm:endpoint_internal_impossibility}. For \(\lambda>1\), amplification can generate such local complementarity cases by moving the pooled logit away from \(0\), toward class \(0\) or class \(1\) depending on its sign. Figure~\ref{fig:exp3_simulation} shows that larger values of \(k_0\) and \(k_1\) are associated with positive global complementarity, although the relation is not deterministic: because cross-entropy is unbounded, a small number of amplified wrong-sign cases can outweigh many locally positive gains.

\begin{figure*}[h!]
\centering
\includegraphics[width=\linewidth]{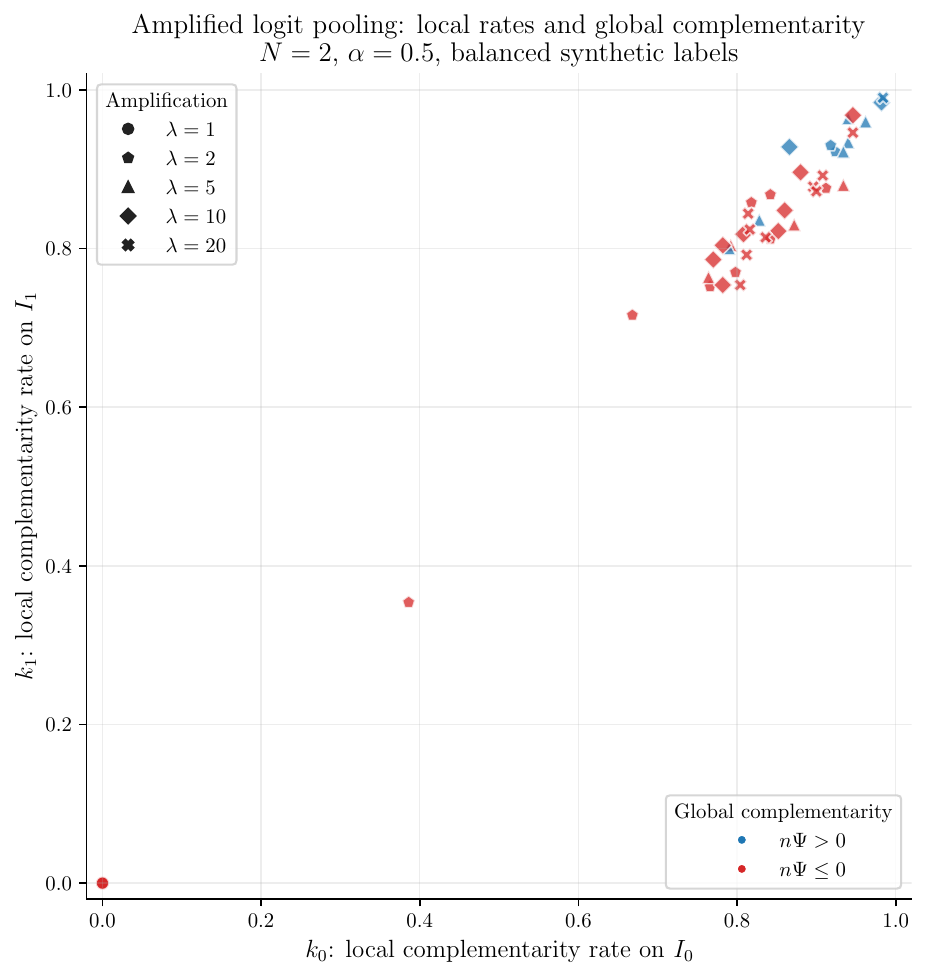}
\caption{
\(N\!=\!2\) binary classification under cross-entropy with amplified logit pooling and fixed \(\alpha=0.5\).
Each point is one simulated pair of probabilistic predictors, plotted by the class-wise rates \((k_0,k_1)\) of canonical local complementarity: for \(y_i=0\), the pooled prediction lies below both input probabilities; for \(y_i=1\), it lies above both input probabilities.
Color indicates global complementarity: blue points satisfy \(n\Psi_{\mathsf T,\diamondsuit}^{m_{2,\alpha,\lambda}^{\mathrm{logit}}}>0\), while red points do not.
The plot illustrates that larger local-complementarity rates are associated with positive global complementarity, although the relation is not deterministic because cross-entropy losses can be dominated by a small number of amplified wrong-sign cases.}
\label{fig:exp3_simulation}
\end{figure*}



\bibliographystyle{plainnat} 
\bibliography{sample-base}

\end{document}